%% file: main_jcy1.tex
\documentclass{article} 
\usepackage{iclr2025_conference,times}

\input{math_commands.tex}

\usepackage{hyperref}
\usepackage{url}
\usepackage{algpseudocode}
\usepackage{algorithm}
\usepackage{subcaption}
\usepackage{mathtools}
\usepackage{multirow}
\usepackage{graphicx}
\usepackage{booktabs}
\usepackage{amsmath}
\usepackage{wrapfig}
\usepackage{pifont}
\usepackage{paralist}
\usepackage{bm}

\newcommand{\cmark}{\ding{51}}%
\newcommand{\xmark}{\ding{55}}%

\usepackage{orcidlink}
\usepackage{color, colortbl} 
\usepackage{rotating}
\definecolor{Gray}{gray}{0.9}

\newcommand{\iisb}{I{\textsuperscript{2}}SB}
\newcommand{\EE}{\mathbb{E}}
\newcommand{\LL}{\mathcal{L}}
\newcommand{\HH}{\bm{H}}
\newcommand{\xx}{\bm{x}}
\newcommand{\yy}{\bm{y}}

\newcommand{\ww}{\bm{w}}

\DeclareMathOperator{\LPIPS}{LPIPS}
\DeclareMathOperator{\GCTM}{GCTM}
\DeclareMathOperator{\CTM}{CTM}
\DeclareMathOperator{\EMA}{EMA}
\DeclareMathOperator{\DSM}{DSM}

\DeclareMathOperator{\FM}{FM}
\DeclareMathOperator{\OT}{OT}
\DeclareMathOperator{\sg}{sg}

\def\eqref#1{(\ref{#1})}
\usepackage{amsthm}

\newtheorem{theorem}{Theorem}
\newcommand{\eg}{\emph{e.g.}}
\newcommand{\ie}{\emph{i.e.}}

\title{Generalized Consistency Trajectory Models for Image Manipulation}

\author{Beomsu Kim\textsuperscript{*} \\ KAIST
\And Jaemin Kim\thanks{Equal Contribution} \\ KAIST
\And Jeongsol Kim \\ KAIST
\And Jong Chul Ye \\ KAIST
}

\iclrfinalcopy 
\begin{document}

\maketitle

\vspace{-3mm}
\begin{abstract}
\vspace{-3mm}
Diffusion models (DMs) excel in unconditional generation, as well as on applications such as image editing and restoration. The success of DMs lies in the iterative nature of diffusion: diffusion breaks down the complex process of mapping noise to data into a sequence of simple denoising tasks. Moreover, we are able to exert fine-grained control over the generation process by injecting guidance terms into each denoising step. However, the iterative process is also computationally intensive, often taking from tens up to thousands of function evaluations. Although consistency trajectory models (CTMs) enable traversal between any time points along the probability flow ODE (PFODE) and score inference with a single function evaluation, CTMs only allow translation from Gaussian noise to data. This work aims to unlock the full potential of CTMs by proposing generalized CTMs (GCTMs), which translate between arbitrary distributions via ODEs. We discuss the design space of GCTMs and demonstrate their efficacy in various image manipulation tasks such as image-to-image translation, restoration, and editing.
\end{abstract}

\vspace{-6mm}
\section{Introduction}

Diffusion-based generative models (DMs) learn the scores of noise-perturbed data distributions, which can be used to translate samples between two distributions by numerically integrating an SDE or a probability flow ODE (PFODE) \citep{ddpm, dhariwal2021diffusion, song2020score}. They have achieved remarkable progress over recent years, even surpassing well-known generative models such as Generative Adversarial Networks (GANs) \citep{gan} or Variational Autoencoders (VAEs) \citep{vae} in terms of sample quality. Moreover, diffusion models have found wide application in areas such as image-to-image translation \citep{saharia2022palette}, image restoration \citep{chung2022dps, cddb}, image editing \citep{meng2021sdedit}, etc.

The success of DMs can largely be attributed to the iterative nature of diffusion, arising from its foundation on differential equations -- multi-step generation grants high-quality image synthesis by breaking down the complex process of mapping noise to data into a composition of simple denoising steps. We are also able to exert fine-grained control over the generation process by injecting minute guidance terms into each step \citep{chung2022dps, ho2022classifier}. Indeed, guidance is an underlying principle behind numerous diffusion-based image editing and restoration algorithms.

However, its iterative nature is also a curse, as diffusion inference often demands from tens to thousands of number of neural function evaluations (NFEs) per sample, rendering practical usage difficult. Consequently, there is now a large body of works on improving the inference speed of DMs. Among them, distillation refers to methods which train a neural network to translate samples along PFODE trajectories generated by a pre-trained teacher DM in one or two NFEs. Representative distillation methods include progressive distillation (PD) \citep{salimans2022progressive}, consistency models (CMs) \citep{song2023consistency}, and consistency trajectory models (CTMs) \citep{kim2023consistency}.

In contrast to PD or CMs which only allow traversal to the terminal point of the PFODE, CTMs enable traversal between any pair of time points along the PFODE as well as score inference, all in a single inference step. Thus, in theory, CTMs are more amenable to guidance, and are applicable to a wider variety of downstream image manipulation tasks. Yet, there is a lack of works exploring the effectiveness of CTMs in such context.

In this work, we take a step towards unlocking the full potential of CTMs. To this end, we first propose generalized CTMs (GCTMs) which generalize the theoretical framework behind CTMs with Flow Matching \citep{yaron2023flow} to enable translation between {\em two arbitrary distributions}. Next, we discuss the design space of GCTMs, and how each design choice influences the downstream task performance. Finally, we demonstrate the power of GCTMs on a variety of image manipulation tasks. Specifically, our contributions can be summarized as follows.
\begin{itemize}
    \item \textbf{Generalization of theory.} We propose GCTMs, which uses conditional flow matching theory to enable one-step translation between two arbitrary distributions (Theorem \ref{theorem:1}). This stands in contrast to CTMs, which is only able to learn PFODEs from Gaussian to data. In fact, we prove CTM is a special case of GCTM when one side is Gaussian (Theorem \ref{theorem:2}).
    \item \textbf{Elucidation of design space.} We clarify the design components of GCTMs, and explain how each component affects downstream task performance (Section \ref{sec:design}). In particular, flexible choice of couplings enable GCTM training in both unsupervised and supervised settings, allowing us to accelerate zero-shot and supervised image manipulation algorithms.
    \item \textbf{Empirical verification.} We demonstrate the potential of GCTMs on unconditional generation, image-to-image translation, image restoration, image editing, and latent manipulation. We show that GCTMs achieve competitive performance even with NFE = 1.
\end{itemize}

\begin{figure}[t]
\centering
\includegraphics[width=0.9\textwidth]{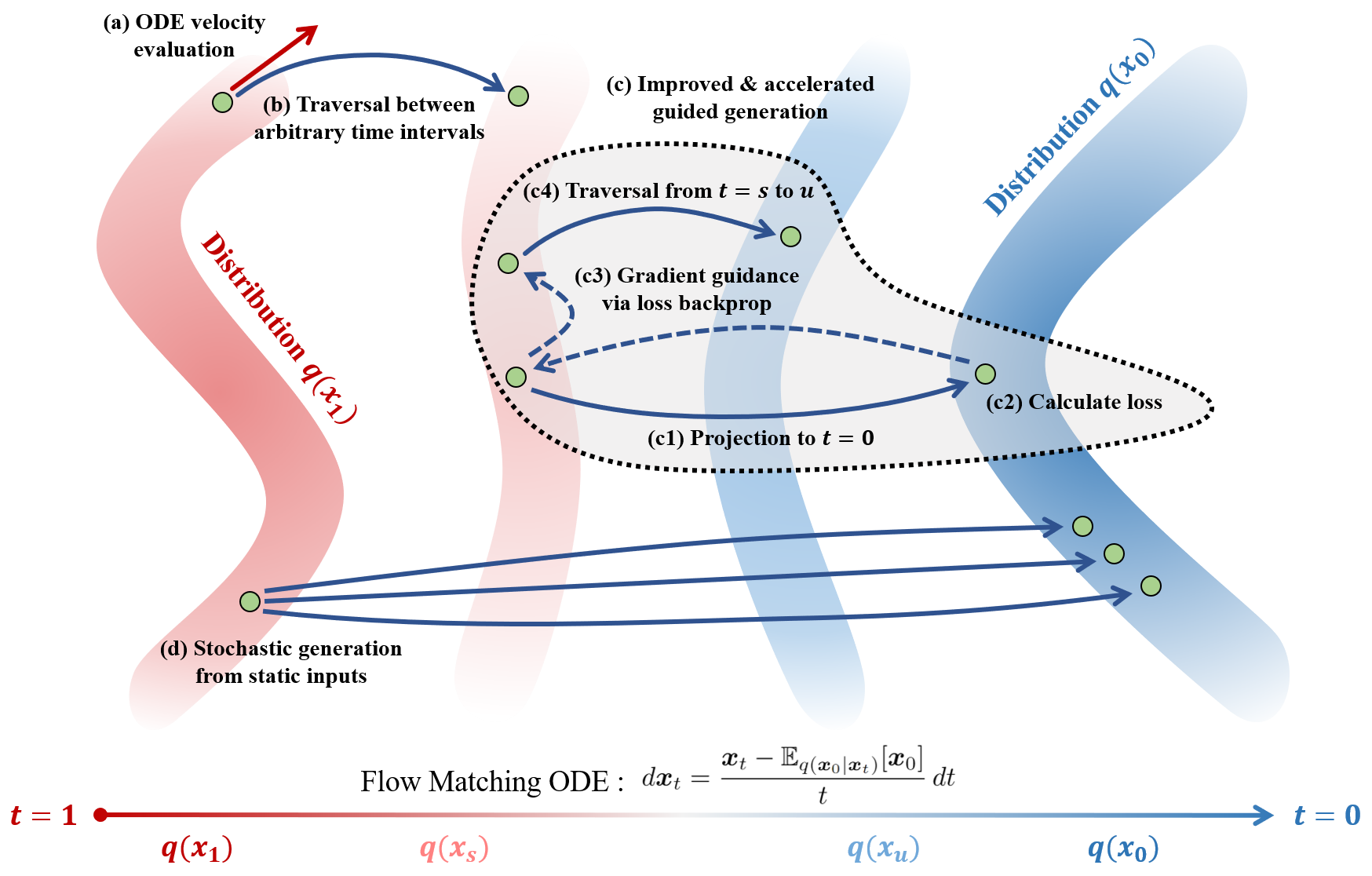}
\caption{\textbf{An illustration of GCTM and its applications -- solid arrows can be implemented by a single forward pass of the GCTM network.} GCTMs learn to traverse the Flow Matching ODE which is capable of {\em interpolating two arbitrary distributions $q(\xx_0)$ and $q(\xx_1)$}. GCTMs allow \textbf{(a)} one-step inference of ODE velocity, \textbf{(b)} one-step traversal between arbitrary time intervals of the ODE, \textbf{(c)} improved gradient-guidance by using exact posterior sample instead of posterior mean, and \textbf{(d)} one-step generation of varying outputs $\xx_0$ given a fixed input $\xx_1$.}
\label{fig:main}
\vspace{-0.5cm}
\end{figure}

\section{Related Work}

\textbf{Diffusion model distillation.} Despite the success of diffusion models (DMs) in generation tasks, DMs require large number of function evaluations (NFEs). As a way to improve the inference speed, the distillation method is proposed to predict the previously trained teacher DM’s output, \eg, score function. Progressive distillation (PD) \citep{salimans2022progressive} progressively reduces the NFEs by training the student model to learn predictions corresponding to two-steps of the teacher model's deterministic sampling path. Consistency models (CMs) \citep{song2023consistency} perform distillation by reducing the self-consistency function over the generative ODE. The above methodologies only consider the output of the ODE path. In contrast, consistency trajectory model (CTM) \citep{kim2023consistency} simultaneously learns the integral and infinitesimal changes of the PFODE trajectory. Our paper extends CTM to learn the PFODE trajectory between two arbitrary distributions.

\noindent
\textbf{Zero-shot image restoration via diffusion.}
Image restoration such as super-resolution, deblurring, and inpainting can be formulated as inverse problems, which obtain true signals from given observations. With the advancements in DMs serving as powerful priors, diffusion based inverse solvers have been explored actively. As a pioneering work, DDRM \citep{kawar2022denoising} performs denoising steps on the spectral space of a linear corrupting matrix. DPS \citep{chung2022dps} and $\mathrm\Pi$GDM \citep{song2022pseudoinverse} propose posterior sampling by estimating the likelihood distribution through Jensen's approximation and Gaussian assumption, respectively. While diffusion-based inverse solvers facilitate zero-shot image restoration, they often necessitate prolonged sampling times.

\noindent
\textbf{Image translation via diffusion.} Conditional GAN-based Pix2Pix \citep{isola2017pix2pix} specifies the task of translating one image into another image as image-to-image translation. SDEdit \citep{meng2021sdedit} avoids mode collapse and learning instabilities with GANs by utilizing DMs to translate edited images along SDEs. Palette \citep{saharia2022palette} proposed conditional DMs for image-to-image translation tasks. To address the Gaussian prior constraint with DMs, Schr\"{o}dinger bridge (SB) or direct diffusion bridge (DDB) methods have been proposed to learn SDEs between arbitrary two distributions \citep{liu20232i2sb,kim2023unsb,indi,cddb}. However, models that follow SDEs often require large NFEs. In contrast, our model learns ODE paths between two arbitrary distributions and demonstrates competitive results with NFE = 1.

\section{Background}

\subsection{Diffusion Models}

Diffusion models learn to reverse the process of corrupting data into Gaussian noise. Formally, the corruption process can be described by a forward SDE
\begin{align}
d\xx_t = \sqrt{2 t} \, d \ww_t \label{eq:forward_sde}
\end{align}
defined on the time interval $t \in (0,T)$. Given $\xx_0$ distributed according to a data distribution $p(\xx_0)$, \eqref{eq:forward_sde} sends $\xx_0$ to Gaussian noise as $t$ increases from $0$ to $T$. The reverse of the corruption process can be described by the reverse SDE
\begin{align}
d\xx_t = -2t \nabla \log p(\xx_t) \, dt + \sqrt{2t} \, d\bar{\ww}_t \label{eq:reverse_sde}
\end{align}
or its deterministic counterpart, the probability flow ODE (PFODE)
\begin{align}
d\xx_t = - t \nabla \log p(\xx_t) \, dt = t^{-1}(\xx_t - \EE_{p(\xx_0 | \xx_t)}[\xx_0]) \, dt \label{eq:pf_ode}
\end{align}
where $p(\xx_t)$ is the distribution of $\xx_t$ following \eqref{eq:forward_sde}, and $\bar{\ww}_t$ is the standard Wiener process in reverse-time. Given a noise sample $\xx_T \sim p(\xx_T)$, $\xx_t$ following \eqref{eq:reverse_sde} or \eqref{eq:pf_ode} is distributed according to $p(\xx_t)$ as $t$ decreases from $T$ to $0$. Thus, diffusion models are able to generate data from noise by approximating the scores $\nabla \log p(\xx_t)$ via score matching, and then numerically integrating \eqref{eq:reverse_sde} or \eqref{eq:pf_ode}.

\subsection{Consistency Trajectory Models (CTMs)}

CTMs learn to translate samples between arbitrary time points of PFODE trajectories, \ie, the goal of CTMs is to learn the integral of the PFODE
\begin{align}
\textstyle G(\xx_t,t,s) \coloneqq \xx_t + \int_t^s u^{-1}(\xx_u - \EE_{p(\xx_0 | \xx_u)}[\xx_0]) \, du \label{eq:G}
\end{align}
for $s \leq t$, where the terminal distribution $p(\xx_T)$ is assumed to be Gaussian. The parametrization
\begin{align}
\textstyle G(\xx_t,t,s) = \frac{s}{t} \xx_t + \left(1 - \frac{s}{t}\right) g(\xx_t,t,s)  \label{eq:G_param}
\end{align}
where
\begin{align}
\textstyle g(\xx_t,t,s) = \xx_t + \frac{t}{t - s} \int_t^s u^{-1}(\xx_u - \EE_{p(\xx_0 | \xx_u)}[\xx_0]) \, du \label{eq:g}
\end{align}
enables both traversal along the PFODE as well as score inference, since
\begin{align}
\lim_{s \rightarrow t} g(\xx_t,t,s) = \EE_{p(\xx_0 | \xx_t)}[\xx_0] \label{eq:g_limit}
\end{align}
so we may define $g(\xx_t,t,t) \coloneqq \EE_{p(\xx_0 | \xx_t)}[\xx_0]$.

Given a pre-trained DM, CTMs approximate $g$ with a neural net $g_\theta$ by simultaneously minimizing a distillation loss and a denoising score matching (DSM) loss. The distillation loss is
\begin{align}
\LL_{\CTM}(\theta) \coloneqq \EE_{0 \leq s \leq u < t \leq T} \EE_{p(\xx_t)} \left[ d\left(G_\theta(\xx_t,t,s), G_{\sg(\theta)}(\xx_{t \rightarrow u},u,s)\right) \right] \label{eq:ctm_loss}
\end{align}
where $G_\theta$ is \eqref{eq:G_param} with $g_\theta$ in place of $g$, $d(\cdot,\cdot)$ is a measure of similarity between inputs, $\sg$ is the stop-gradient operation, and $\xx_{t \rightarrow u}$ is defined to be the integral of PFODE from time $t$ to $u$ starting from $\xx_t$ using score estimates from the pre-trained diffusion model. Minimization of \eqref{eq:ctm_loss} causes $G_\theta$ to adhere to PFODE trajectories generated by the pre-trained diffusion model. The DSM loss is
\begin{align}
\LL_{\DSM}(\theta) \coloneqq \EE_{0 \leq \hat{t} \leq T} \EE_{p(\xx_0) p(\xx_1)} \EE_{p(\xx_{\hat{t}} | \xx_0, \xx_T)} \left[ \|\xx_0 - g_\theta(\xx_{\hat{t}},\hat{t},\hat{t})\|_2^2 \right] \label{eq:dsm_loss}
\end{align}
where $p(\xx_t | \xx_0, \xx_T) = \delta_{\xx_0 + t \xx_T}(\xx_t)$, and minimization of \eqref{eq:dsm_loss} causes $g_\theta$ to satisfy \eqref{eq:g_limit}. This loss acts as a regularization which improves score accuracy, and is crucial for sampling with large NFEs. Thus, the final objective is
\begin{align}
\LL_{\CTM}(\theta) + \lambda_{\DSM} \LL_{\DSM}(\theta),
\end{align}
and it is possible to further improve sample quality by adding a GAN loss.

\subsection{Flow Matching (FM)}

Flow Matching is another technique for learning PFODEs between two distributions $q(\xx_0)$ and $q(\xx_1)$. Specifically, let $q(\xx_0,\xx_1)$ be a joint distribution of $\xx_0$ and $\xx_1$. Define
\begin{align}
q(\xx_t | \xx_0, \xx_1) \coloneqq \delta_{(1 - t) \xx_0 + t \xx_1}(\xx_t), \quad q(\xx_t) \coloneqq \EE_{q(\xx_0,\xx_1)}[q(\xx_t | \xx_0, \xx_1)]
\end{align}
where $t \in (0,1)$ and $\delta_{\yy}(\cdot)$ is a Dirac delta at $\yy$. Then the ODE given by
\begin{align}
d\xx_t = \EE_{q(\xx_0, \xx_1 | \xx_t)}[\xx_1 - \xx_0] \, dt \label{eq:flow_ode}
\end{align}
generates the probability path $q(\xx_t)$, \ie, with terminal condition $\xx_1 \sim q(\xx_1)$, $\xx_t$ following \eqref{eq:flow_ode} is distributed according to $q(\xx_t)$. Analogous to denoising score matching, the velocity term in \eqref{eq:flow_ode} can be approximated by a neural network $\vv_\phi$ which solves a regression problem
\begin{align}
\min_{\phi} \EE_{q(\xx_0,\xx_1,\xx_t)} \left[ \| (\xx_1 - \xx_0) - \vv_\phi(\xx_t,t) \|_2^2 \right].
\end{align}
Unlike diffusion whose terminal distribution $p(\xx_T)$ is Gaussian, $q(\xx_1)$ can be arbitrary. Also, we remark that the theory presented here is only a particular instance of FM called conditional FM.

\section{Generalized Consistency Trajectory Models (GCTMs)}

We now present GCTMs, which generalize CTMs to enable translation between arbitrary distributions. We begin with a crucial theorem which proves we can parametrize the solution to the FM ODE \eqref{eq:flow_ode} in a form analogous to CTMs. The proof is deferred to Appendix \ref{proof:theorem1}.
\begin{theorem} \label{theorem:1}
The ODE \eqref{eq:flow_ode} is equivalent to
\begin{align}
d\xx_t = t^{-1}(\xx_t - \EE_{q(\xx_0 | \xx_t)}[\xx_0]) \, dt \label{eq:flow_ode_equiv}
\end{align}
defined on $t \in (0,1)$. Hence, we can express the solution to \eqref{eq:flow_ode} as
\begin{gather}
\textstyle G(\xx_t,t,s) = \frac{s}{t} \xx_t + \left(1 - \frac{s}{t}\right) g(\xx_t,t,s)  \label{eq:flow_G_param} \\
\textstyle \textit{where} \quad g(\xx_t,t,s) = \xx_t + \frac{t}{t - s} \int_t^s u^{-1}(\xx_u - \EE_{q(\xx_0 | \xx_u)}[\xx_0]) \, du. \label{eq:flow_g}
\end{gather}
\end{theorem}
There are two differences between \eqref{eq:G_param} and \eqref{eq:flow_G_param}. First, the time variables $t$ and $s$ now lie in the unit interval $(0,1)$ instead of $(0,T)$, and second, $p(\xx_0 | \xx_u)$ is replaced with $q(\xx_0 | \xx_u)$.
The second difference is what enables translation between arbitrary distributions, as $q(\xx_0 | \xx_u)$ recovers clean images $\xx_0$ given images $\xx_u$ perturbed by arbitrary type of vectors (e.g., Gaussian noise, images, etc.), while $p(\xx_0 | \xx_u)$ recovers clean images $\xx_0$ only for Gaussian-perturbed samples $\xx_u$. We call a neural network $g_\theta$ which approximates \eqref{eq:flow_g} a GCTM, and we can train such a network by optimizing the FM counterparts of $\LL_{\CTM}$ and $\LL_{\DSM}$:
\begin{align}
\LL_{\GCTM}(\theta) \coloneqq \EE_{0 \leq s \leq u < t \leq 1} \EE_{q(\xx_t)} \left[ d\left(G_\theta(\xx_t,t,s), G_{\sg(\theta)}(\xx_{t \rightarrow u},u,s)\right) \right] \label{eq:gctm_loss}
\end{align}
where $G_\theta$ is \eqref{eq:flow_G_param} with $g$ replaced by $g_\theta$, and
\begin{align}
\LL_{\FM}(\theta) \coloneqq \EE_{0 \leq \hat{t} \leq 1} \EE_{q(\xx_0,\xx_1)} \EE_{q(\xx_{\hat{t}} | \xx_0,\xx_1)} \left[ \|\xx_0 - g_\theta(\xx_{\hat{t}},\hat{t},\hat{t})\|_2^2 \right]. \label{eq:gctm_dsm_loss}
\end{align}

The next theorem shows that the PFODE \eqref{eq:pf_ode} learned by CTMs is a special case of the ODE \eqref{eq:flow_ode_equiv} learned by GCTMs, so GCTMs indeed generalize CTMs.  The proof is deferred to Appendix \ref{proof:theorem2}.

\begin{theorem} \label{theorem:2}
Consider the choice of $q(\xx_0,\xx_1) = p(\xx_0) \cdot \mathcal{N}(\xx_1 | \bm{0},\bm{I})$. Let
\begin{align}
t' \coloneqq t/(1+t), \qquad \bar{\xx}_{t'} \coloneqq \xx_{t} / (1+t) \label{eq:change}
\end{align}
where $t \in (0,\infty)$ and $\xx_t$ follows the PFODE \eqref{eq:pf_ode}.
Then
\begin{align}
\EE_{p(\xx_0 | \xx_t)}[\xx_0] = \EE_{q(\xx_0 | \bar{\xx}_{t'})}[\xx_0] \label{eq:score_equiv}
\end{align}
and $\bar{\xx}_{t'}$ follows the ODE
\begin{align}
d\bar{\xx}_{t'} = {t'}^{-1} (\bar{\xx}_{t'} - \EE_{q(\xx_0 | \bar{\xx}_{t'})}[\xx_0]) \, dt' \label{eq:ode_equiv}
\end{align}
on $t' \in (0,1)$.
Furthermore, let $G_{\CTM}$, $g_{\CTM}$ denote CTM solutions and let $G_{\GCTM}$, $g_{\GCTM}$ denote GCTM solutions. Then with $s' = s / (1 + s)$,
\begin{align}
\begin{cases}
G_{\CTM}(\xx_t,t,s) = G_{\GCTM}(\bar{\xx}_{t'},t',s') \cdot (1 + s) \\
g_{\CTM}(\xx_t,t,t) = g_{\GCTM}(\bar{\xx}_{t'},t',t')
\end{cases} \label{eq:equality}
\end{align}
\end{theorem}

In short, \eqref{eq:score_equiv} shows the equivalence of scores, and \eqref{eq:ode_equiv} shows the equivalence of ODEs. Thus, given $g_\theta$ trained with $\LL_{\FM}$ and $\LL_{\GCTM}$ with the setting of Thm. \ref{theorem:1}, we are able to evaluate diffusion scores and simulate diffusion PFODE trajectories with a simple change of variables \eqref{eq:change}, as shown in \eqref{eq:equality}.

Given GCTM's capability to replicate CTM, we will now outline the key components of GCTM that enable its significant extension for various downstream tasks. This flexibility offers a notable advantage of GCTM over CTM.

\subsection{The Design Space of GCTMs} \label{sec:design}

\noindent
\textbf{Coupling $q(\xx_0,\xx_1)$.} In contrast to diffusion which only uses the trivial coupling $q(\xx_0,\xx_1) = q(\xx_0)q(\xx_1)$ in $\LL_{\DSM}(\theta)$, FM allows us to use arbitrary joint distributions of $q(\xx_0)$ and $q(\xx_1)$ in $\LL_{\FM}(\theta)$. Intuitively, $q(\xx_0,\xx_1)$ encodes our inductive bias for what kind of pairs $(\xx_0,\xx_1)$ we wish the model to learn, since FM ODE is distributed $q(\xx_t)$ at each time $t$, and $q(\xx_t)$ is the distribution of $(1-t) \xx_0 + t \xx_1$ for $(\xx_0,\xx_1) \sim q(\xx_0,\xx_1)$. Here, we list three valid couplings of GCTM as examples (see Alg. \ref{alg:sampling} for code). In contrast,
CTM
only use a special case of independence coupling.
\begin{itemize}
\item \textit{Independent coupling}:
\begin{align}
\label{indepcoup}
q(\xx_0,\xx_1) = q(\xx_0)q(\xx_1)
\end{align}
This coupling reflects no prior assumption about the relation between $\xx_0$ and $\xx_1$. As shown earlier, diffusion models use this type of coupling.
\item \textit{(Entropy-regularized) Optimal transport coupling}:
\begin{align}
q = \argmin_{\Tilde{q}} \EE_{\Tilde{q}(\xx_0,\xx_1)} \left[ \| \xx_0 - \xx_1 \|_2^2 \right] - \tau H(\Tilde{q})
\end{align}
where $\argmin$ is over all joint distributions $\Tilde{q}$ of $q(\xx_0)$ and $q(\xx_1)$, $H$ denotes entropy, and $\tau$ is the regularization coefficient. This coupling reflects the inductive bias that $\xx_0$ and $\xx_1$ must be close together under the Euclidean distance. In practice, we use the Sinkhorn-Knopp (SK) algorithm \citep{cuturi2013sinkhorn} to sample OT pairs. A pseudo-code for SK is given as Alg. \ref{alg:sk} in Appendix \ref{append:algo}.
\item \textit{Supervised coupling}:
\begin{align}
\label{supcoup}
q(\xx_0,\xx_1) = \int q(\xx_0) q(\HH | \xx_0) \delta_{\HH \xx_0}(\xx_1) \, d\HH
\end{align}
where $\HH \sim q(\HH | \xx_0)$ is a random operator, possibly dependent on $\xx_0$, which maps ground-truth data $\xx_0$ to observations $\xx_1$, \ie, $\xx_1 = \HH \xx_0$. For instance, in the context of learning an inpainting model, $\HH$ is could be a random masking operator. For a fixed $\HH$, $q(\HH | \xx_0)$ reduces to a Dirac delta. With this coupling, the ODE \eqref{eq:flow_ode_equiv} tends to map observed samples $\xx_1$ to ground-truth data $\xx_0$ as $t \rightarrow 0$.
\end{itemize}

\begin{table}[t]
\vspace{-0.4cm}
\scalebox{0.83}{
\begin{minipage}[b]{0.49\linewidth}
\begin{algorithm}[H]
\caption{$q(\xx_0,\xx_1)$ Sampling}
\label{alg:sampling}
\begin{algorithmic}[1]
\State \textbf{Assume} $m = 1, \ldots, M$, Batch size $M$
\If{Coupling is \texttt{Independent}}
\State $\{\xx_0^m\}_m \sim q(\xx_0)$, $\{\xx_1^m\}_m \sim q(\xx_1)$
\State \textbf{Return} $\{(\xx_0^m,\xx_1^m)\}_m$
\ElsIf{Coupling is \texttt{OT}}
\State $\{\xx_0^m\}_m \sim q(\xx_0)$, $\{\xx_1^m\}_m \sim q(\xx_1)$
\State \textbf{Return} $\text{SK}(\{\xx_0^m\}_m, \{\xx_1^m\}_m, \tau)$
\ElsIf{Coupling is \texttt{Supervised}}
\State $\{\xx_0^m\}_m \sim q(\xx_0)$, $\HH^m \sim q(\HH | \xx_0^m)$
\State \textbf{Return} $\{(\xx_0^m,\HH^m \xx_0^m)\}_m$
\EndIf
\end{algorithmic}
\end{algorithm}
\end{minipage}}
\hfill
\scalebox{0.85}{
\begin{minipage}[b]{.65\textwidth}
\begin{algorithm}[H]
\caption{GCTM Training}
\label{alg:training}
\begin{algorithmic}[1]
\While{training}
\State Sample times $\{\hat{t}^m\}_m$, $\{(t^m, s^m, u^m)\}_m$
\State With Alg. \ref{alg:sampling}, $\{(\xx_0^m,\xx_1^m)\}_m \sim q(\xx_0,\xx_1)$
\State $\xx_{\hat{t}^m}^m \leftarrow (1 - \hat{t}^m) \xx_0^m + \hat{t}^m \xx_1^m$
\State $\xx_{t^m}^m \leftarrow (1 - t^m) \xx_0^m + t^m \xx_1^m$
\State $\LL_{\FM}(\theta) = \frac{1}{M} \sum_m \|\xx_0^m - g_\theta(\xx_{\hat{t}^m}^m,\hat{t}^m,\hat{t}^m) \|_2^2$
\State $\widetilde{\xx}_{s^m}^m \leftarrow G_{\sg(\theta)}(\xx_{t^m \rightarrow u^m}^m,u^m,s^m)$
\State $\LL_{\GCTM}(\theta) = \frac{1}{M} \sum_{m=1}^M d(G_\theta(\xx_{t^m}^m,t^m,s^m),\widetilde{\xx}_{s^m}^m)$
\State Minimize $\LL_{\GCTM}(\theta) + \lambda_{\FM} \LL_{\FM}(\theta)$
\EndWhile
\end{algorithmic}
\end{algorithm}
\end{minipage}}
\end{table}

\noindent
\textbf{Gaussian perturbation.} The cardinality of the support of $q(\xx_1)$ must be larger than or equal to the cardinality of the support of $q(\xx_0)$ for there to be a well-defined ODE from $q(\xx_1)$ to $q(\xx_0)$. This is because the ODE trajectory given an initial condition is unique, so a single sample $\xx_1 \sim q(\xx_1)$ cannot be transported to multiple points in the support of $q(\xx_0)$. A simple way to address this problem is to add small Gaussian noise to $q(\xx_1)$ samples such that $q(\xx_1)$ is supported everywhere.

We emphasize that Gaussian perturbation allows GCTMs to achieve one-to-many generation when we use the supervised coupling. Concretely, consider the scenario where there are multiple labels $\xx_0 \sim q(\xx_0 | \xx_1)$ which correspond to an observed $\xx_1$. Then, the perturbation $\bm{\epsilon}$ added to $\xx_1$ acts as a source of randomness, allowing the GCTM network to map $\xx_1 + \bm{\epsilon}$ to distinct labels $\xx_0$ for distinct $\bm{\epsilon}$. This stands in contrast to simply regressing the neural network output of $\xx_1$ to corresponding labels $\xx_0 \sim q(\xx_0 | \xx_1)$ with $\ell_2$ loss, as this will cause the network to map $\xx_1$ to the blurry posterior mean $\EE_{q(\xx_0|\xx_1)}[\xx_0]$ instead of a sharp image $\xx_0$. Indeed, in Section \ref{sec:i2i}, we observe blurry outputs if we use regression instead of GCTMs.

\section{Experiments} \label{sec:experiments}

We now explore the possibilities of GCTMs on unconditional generation, image-to-image translation, image restoration, image editing, and latent manipulation. In particular, GCTM admits NFE = 1 sampling via $\xx_t \mapsto G_\theta(\xx_t,t,0)$. Due to the similarities between CTMs and GCTMs as detailed in Thm. \ref{theorem:1}, GCTMs can be trained using CTM training methods. In fact, we run Alg. \ref{alg:training} with the method in Section 5.2 of \citep{kim2023consistency} to train all GCTMs without pre-trained teacher models. A complete description of training settings are deferred to Appendix \ref{append:settings}.

\begin{figure}[t]
\centering
\begin{subfigure}{0.32\linewidth}
\includegraphics[width=1.0\linewidth]{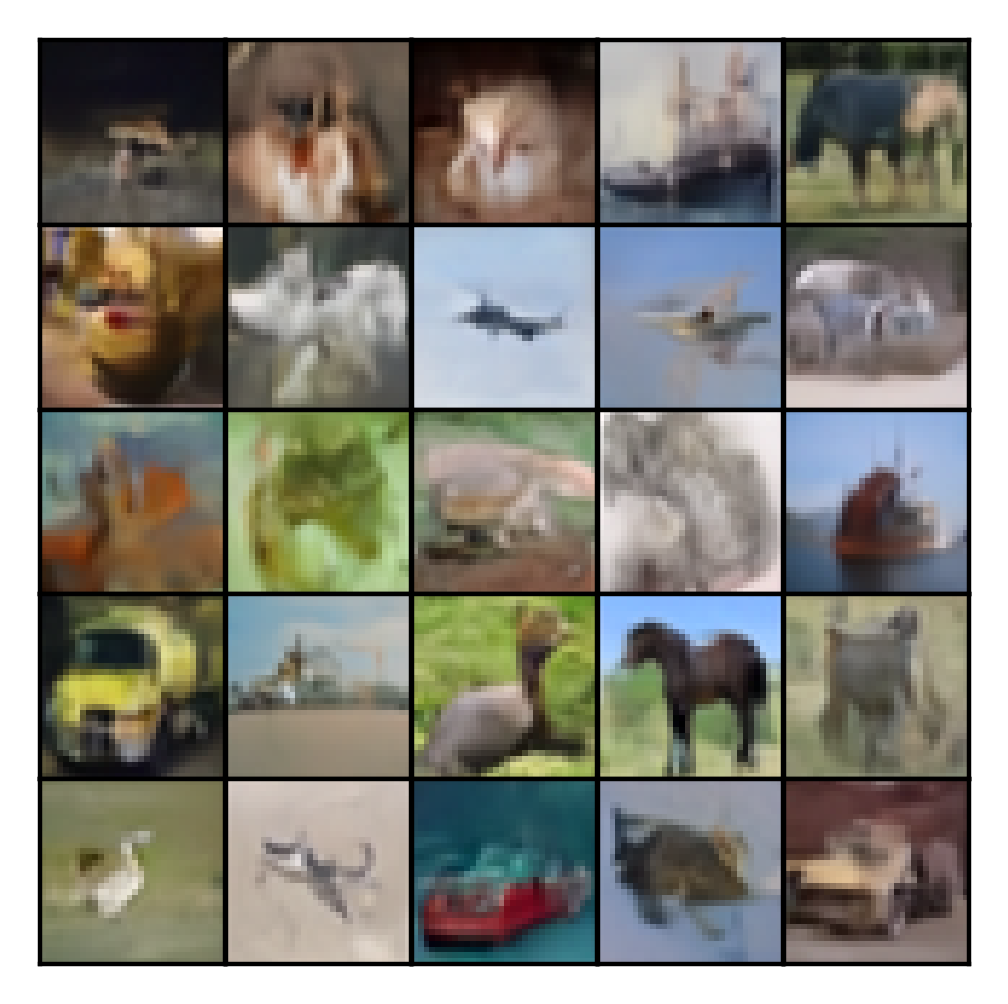}
\caption{Indep. $N = 4$, FID = 24.7}
\end{subfigure}
\hfill
\begin{subfigure}{0.32\linewidth}
\includegraphics[width=1.0\linewidth]{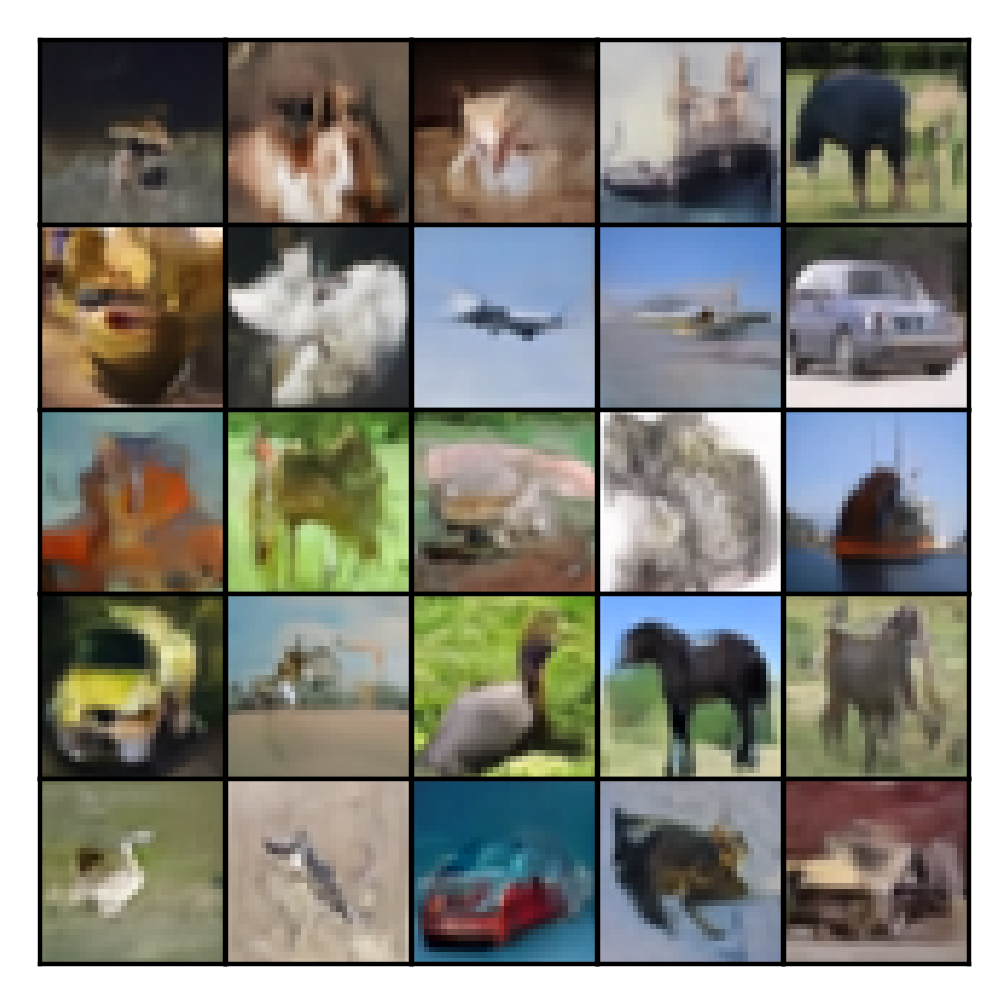}
\caption{OT $N = 4$, FID = 18.2}
\end{subfigure}
\hfill
\begin{subfigure}{0.32\linewidth}
\includegraphics[width=1.0\linewidth]{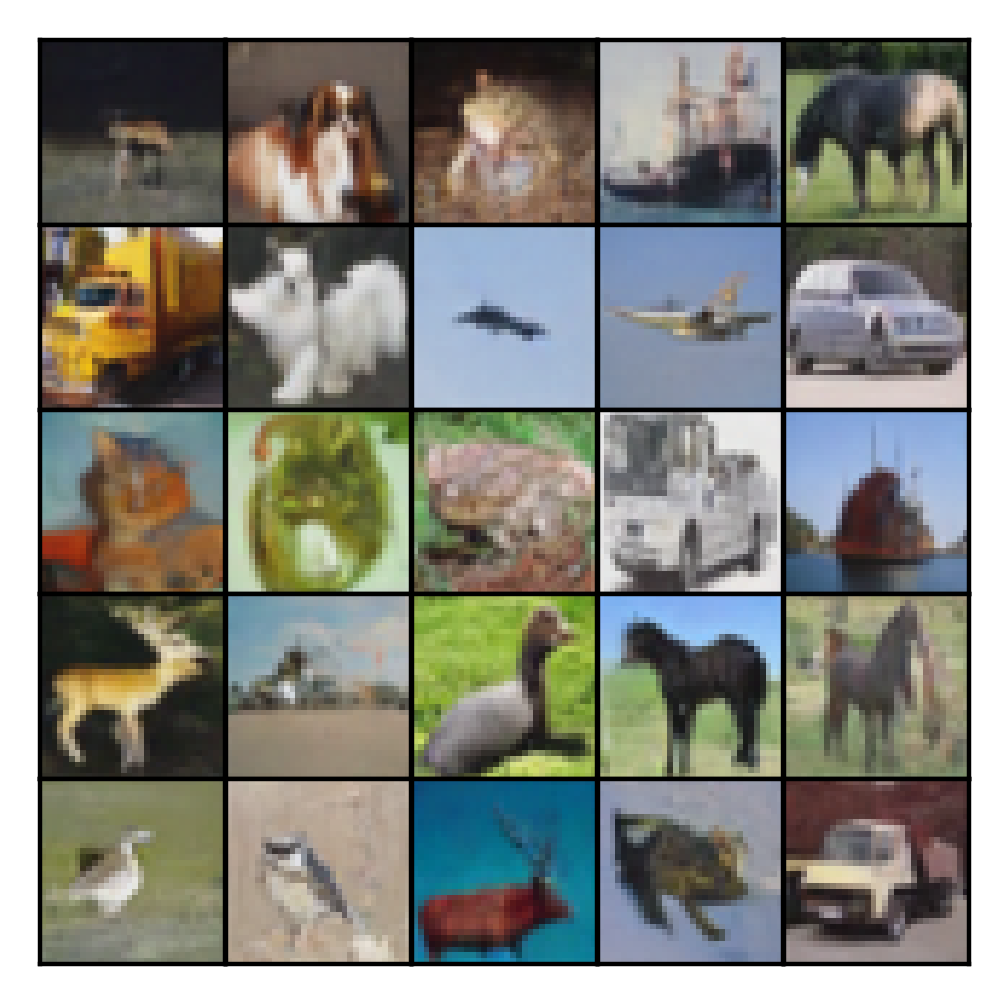}
\caption{OT $N = 32$, FID = 5.32}
\end{subfigure}
\vspace{-0.1cm}
\caption{CIFAR10 unconditional samples with NFE = 1.}
\label{fig:cifar10_samples}
\end{figure}

\subsection{Fast Unconditional Generation} \label{sec:unconditional}

In the scenario where we do not have access to data pairs, we must resort to either the independent coupling or the OT coupling. Here, we show that the optimal transport coupling can significantly accelerate the convergence speed of GCTMs during training, especially when we use a smaller number of timesteps $N$.

\begin{wraptable}{r}{0.4\linewidth}
\begin{minipage}[b]{1.0\linewidth}
\centering
\includegraphics[width=0.75\linewidth]{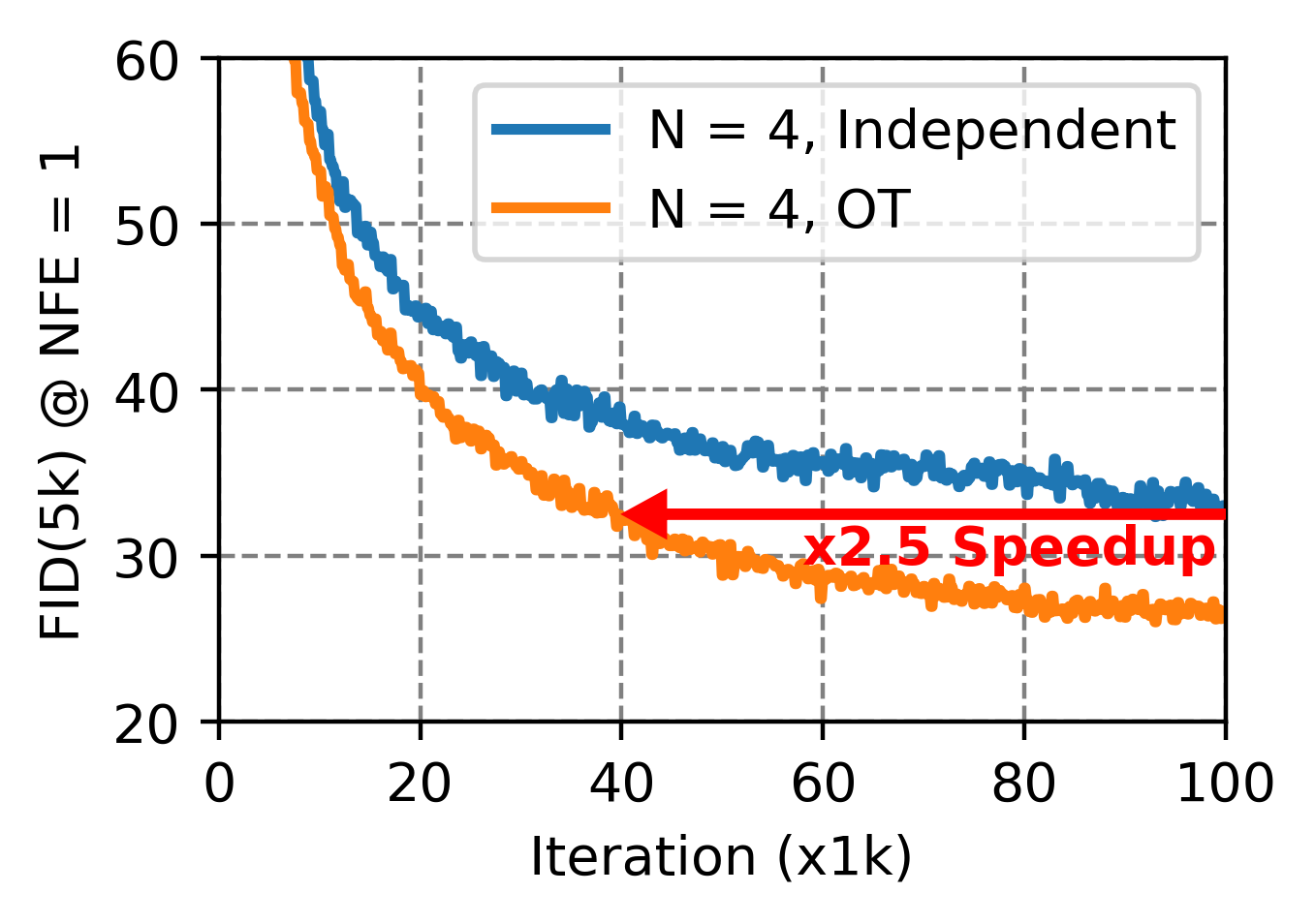}
\vspace{-0.3cm}
\captionof{figure}{Training acceleration.}
\label{fig:cifar10_accel}
\vspace{0.2cm}
\end{minipage}
\begin{minipage}[b]{1.0\linewidth}
\centering
\resizebox{0.7\columnwidth}{!}{%
\begin{tabular}{ccc}
\toprule
\textbf{Method} & \textbf{Teacher  } & \textbf{FID} $\downarrow$ \\
\cmidrule{1-3}
CTM & \cmark & 5.28 \\
    & \xmark & 9.00 \\
CM  & \cmark & 3.55 \\
    & \xmark & 8.70 \\
iCM & \xmark & 2.51 \\
\rowcolor{Gray}
GCTM (OT) & \xmark & 5.32 \\
\bottomrule
\end{tabular}}
\caption{FID at NFE = 1.}
\label{table:cifar10_fid}
\end{minipage}
\vspace{-1.0cm}
\end{wraptable}

Using small $N$ may be of interest when we wish to trade-off training speed for performance, since per-iteration training cost of GCTMs increases linearly with $N$. For instance, when $t = 1$ and $u = s = 0$ in the GCTM loss \eqref{eq:gctm_loss}, we need to integrate along the entire time interval $(0,1)$, which requires $N$ steps of ODE integration.

In Figure \ref{fig:cifar10_accel}, we observe up to $\times$2.5 acceleration in terms of training iterations when we use OT coupling instead of independent coupling. Indeed, in Figure \ref{fig:cifar10_samples}, OT coupling samples are visually sharper than independent coupling samples. We postulate this is because (1) OT coupling leads to straighter ODE trajectories, so we can accurately integrate ODEs with smaller $N$, and (2) lower variance from OT pairs leads to smaller variance in loss gradients, as discussed in \citep{pooladian2023multisample}.

In Table \ref{table:cifar10_fid}, we compare the Fr\'echet Inception Distance (FID) \citep{heusel2017fid} of GCTM and relevant baselines on CIFAR10 with NFE = 1. In the setting where we do not use a pre-trained teacher diffusion model, GCTM with OT coupling outperforms all methods with the exception of iCM \citep{song2024icm}, which is an improved variant of CM. Moreover, GCTM is on par with CTM trained with a teacher. We speculate that further fine-tuning of hyper-parameters could push the performance of GCTMs to match that of iCMs, and we leave this for future work.

\subsection{Fast Image-to-Image Translation} \label{sec:i2i}

\begin{table}[t]
\vspace{-0.5cm}
\begin{center}
\renewcommand{\arraystretch}{1}
\resizebox{1.0\textwidth}{!}
{\small
\begin{tabular}{lccccccccccc}
\toprule
\multicolumn{1}{c}{\multirow{2}{*}{Method}} & \multicolumn{1}{c}{\multirow{2}{*}{NFE}} \ & \multicolumn{1}{c}{\multirow{2}{*}{
\begin{tabular}{@{}c@{}}
     Time\\
     (ms)
\end{tabular}
 }} & \multicolumn{3}{c}{Edges$\rightarrow$Shoes} &  \multicolumn{3}{c}{Night$\rightarrow$Day} & \multicolumn{3}{c}{Facades}
  \\ \cmidrule(lr){4-6} \cmidrule(lr){7-9} \cmidrule(lr){10-12}
\multicolumn{1}{c}{} & \multicolumn{1}{c}{} & \multicolumn{1}{c}{} &
  \multicolumn{1}{c}{FID $\downarrow$} &
  \multicolumn{1}{c}{IS $\uparrow$ } &
  \multicolumn{1}{c}{LPIPS $\downarrow$ } &
  \multicolumn{1}{c}{FID $\downarrow$} &
  \multicolumn{1}{c}{IS $\uparrow$ } &
  \multicolumn{1}{c}{LPIPS $\downarrow$ } &
  \multicolumn{1}{c}{FID $\downarrow$} &
  \multicolumn{1}{c}{IS $\uparrow$ } &
  \multicolumn{1}{c}{LPIPS $\downarrow$ } \\
\cmidrule{1-12}
Regression & 1  & 87 
        & 54.3 & \underline{3.41} & \underline{0.100} 
        & 189.2 & \underline{1.85} & \underline{0.373}
        & \underline{121.8} & \textbf{3.28} & 0.274 \\
Pix2Pix \citep{isola2017pix2pix} & 1  & 33 & 77.0 & 3.17 & 0.208 & 158.0 & 1.68 & 0.418 & 134.1 & 2.74 & 0.288 \\
Palette \citep{saharia2022palette} & 5  & 166 
& 334.1 & 1.90 & 0.861
& 350.2 & 1.16 & 0.707
& 259.3 & 2.47 & 0.394 \\
\iisb \citep{liu20232i2sb}
        & 5  & 284 & \underline{53.9} & 3.23 & 0.154 & \textbf{145.8} & 1.79 & 0.376 & 135.2 & 2.51 & \underline{0.269}\\
\rowcolor{Gray}
GCTM    & 1  & 87
        & \textbf{40.3} & \textbf{3.54}  & \textbf{0.097}
        & \underline{148.8} & \textbf{2.00} & \textbf{0.317} 
        &\textbf{ 111.3} & \underline{2.99} & \textbf{0.230}\\
\bottomrule
\end{tabular}}
\end{center}
\caption{Quantitative evaluation of I2I translation. Best is in \textbf{bold}, second best is \underline{underlined}.}
\label{table:i2i}
\end{table}

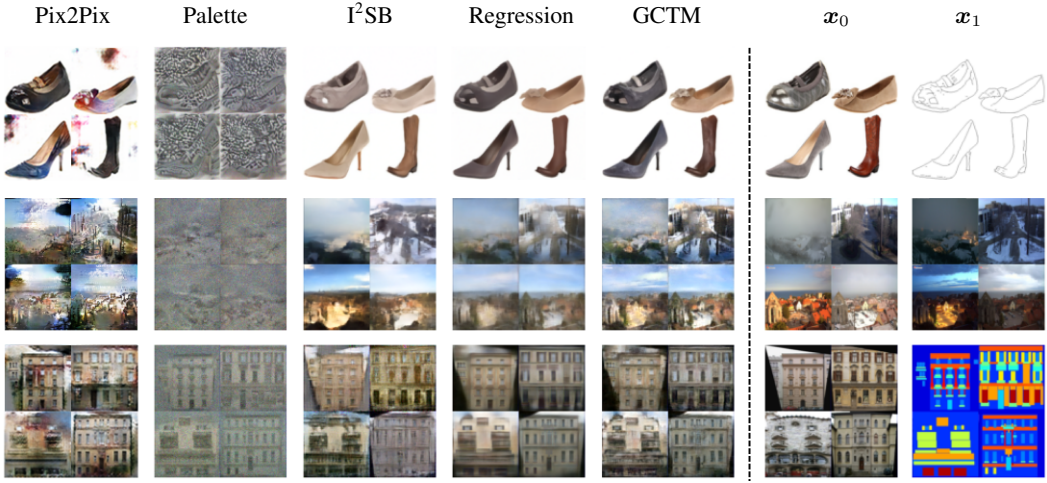
\begin{figure}[t]
\centering
\small{
\includegraphics[width = \textwidth]{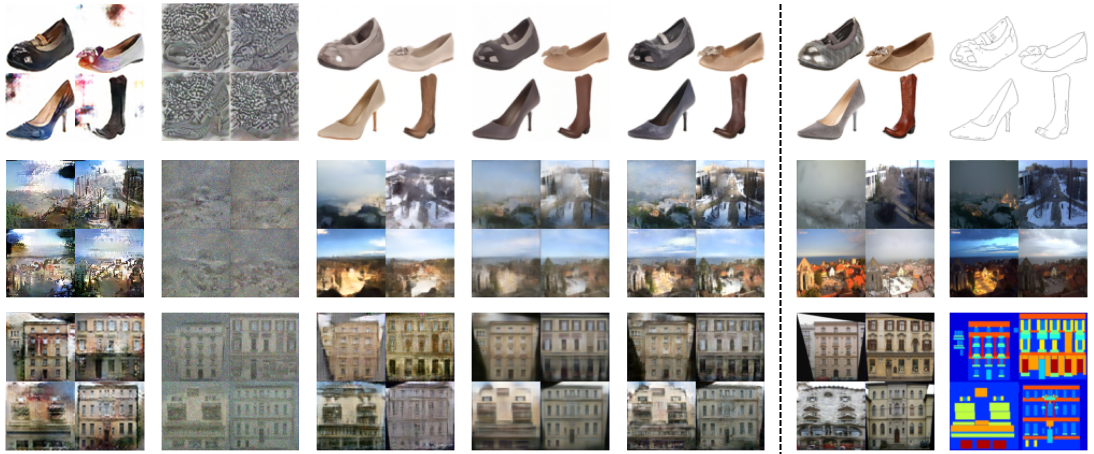}
\put(-384,173){Pix2Pix}
\put(-328,173){Palette}
\put(-266,173){\iisb }
\put(-220,173){Regression}
\put(-158,173){GCTM}
\put(-86,173){$\xx_0$}
\put(-36,173){$\xx_1$}
}
\vspace{-0.5cm}
\caption{Qualitative evaluation of image-to-image translation on Edges$\rightarrow$Shoes (top), Night$\rightarrow$Day (middle) and Facades (bottom). NFE = 5 for I$^2$SB and Palette.}
\vspace{-0.5cm}
\label{fig:i2i}
\end{figure}

Unlike previous distillation methods such as CM or CTM, GCTM can learn ODEs between arbitrary distributions, enabling image-to-image translation. To numerically validate this theoretical improvement, we train GCTMs on three translation tasks Edges$\rightarrow$Shoes, Night$\rightarrow$Day, and Facades \citep{isola2017pix2pix}, scaled to 64$\times$64, with the supervised coupling. We consider three baseline methods: $\ell_2$-regression, Pix2Pix \citep{isola2017pix2pix}, Palette \citep{saharia2022palette} and \iisb \citep{liu20232i2sb}. To evaluate translation performance, we use FID and Inception Score (IS) \citep{barratt2018inception} to rate translation quality and LPIPS \citep{zhang2018lpips} to assess faithfulness to input. We control NFEs such that all methods have similar inference times, and we calculate all metrics on validation samples.

In Table \ref{table:i2i}, we see GCTM shows strong performance on all tasks. In particular, GCTM is good at preserving input structure, as supported by low LPIPS values.
SDE-based methods \iisb \ and Palette show poor performance at low NFEs, even when trained with pairs.
Qualitative results in Figure \ref{fig:i2i} are in line with the metrics. Baselines produce blurry or nonsensical samples, while GCTM produces sharp and realistic images that are faithful to the input.

\begin{table}[t]
\vspace{-0.3cm}
\begin{center}
\renewcommand{\arraystretch}{1}
\resizebox{\textwidth}{!}
{\small
\begin{tabular}{ccccccccccccc}
\toprule
& \multirow{2}{*}{Method} & \multirow{2}{*}{NFE} & \multirow{2}{*}{
\begin{tabular}{@{}c@{}}
     Time\\
     (ms)
\end{tabular}
 }
 &\multicolumn{3}{c}{SR2 - Bicubic} &  \multicolumn{3}{c}{Deblur - Gaussian} & \multicolumn{3}{c}{Inpaint - Center}
  \\ \cmidrule(lr){5-7} \cmidrule(lr){8-10} \cmidrule(lr){11-13}
& \multicolumn{1}{c}{} & \multicolumn{1}{c}{} & \multicolumn{1}{c}{} &
  \multicolumn{1}{c}{PSNR $\uparrow$} &
  \multicolumn{1}{c}{SSIM $\uparrow$ } &
  \multicolumn{1}{c}{LPIPS $\downarrow$ } &
  \multicolumn{1}{c}{PSNR $\uparrow$} &
  \multicolumn{1}{c}{SSIM $\uparrow$ } &
  \multicolumn{1}{c}{LPIPS $\downarrow$ } &
  \multicolumn{1}{c}{PSNR $\uparrow$} &
  \multicolumn{1}{c}{SSIM $\uparrow$ } &
  \multicolumn{1}{c}{LPIPS $\downarrow$ } \\
\cmidrule{1-13}
& DPS & 32 & 1079
    & 31.19 & 0.935 & \underline{0.015}
    & 27.88 & 0.878 & 0.041
    & \textbf{24.69} & \textbf{0.876} & \textbf{0.042}\\
& CM  & 32 & 1074
    & 30.80 & 0.930 & \textbf{0.010}
    & 27.85 & 0.871 & \textbf{0.027}
    & 23.02 & 0.857 & 0.050\\
\rowcolor{Gray}
\cellcolor{white} \multirow{-3}{0.5cm}{\rotatebox[origin=c]{90}{\textit{\textbf{0-Shot}}}} & GCTM & 32 & 1382
    & \textbf{31.61} & \textbf{0.939} & \underline{0.015} 
    & \textbf{28.19} & \textbf{0.885} & \underline{0.037}
    & \underline{24.47} & \textbf{0.876} & \textbf{0.042}\\
\cmidrule{1-13}
 & Regression & 1 & 87 & \textbf{33.46} & \textbf{0.964} & \underline{0.015} & \textbf{31.19} & \textbf{0.942} & \underline{0.015} & \textbf{28.76} & \textbf{0.922} & \underline{0.028} \\

& Palette & 5 & 166
& 17.88 & 0.556 & 0.234
& 17.81 & 0.571 & 0.234
& 16.12 & 0.489 & 0.357\\
 
& \iisb 
    & 5 & 284
    & 26.74 & 0.869 & 0.033 
    & 26.20 & 0.853 & 0.038 
    & 26.01 & 0.874 & 0.038 \\

\rowcolor{Gray}
\cellcolor{white} \multirow{-4}{0.5cm}{\rotatebox[origin=c]{90}{\textit{\textbf{Superv.}}}} & GCTM & 1 & 87 & \underline{32.37} & \underline{0.954} & \textbf{0.009} & \underline{30.56} & \underline{0.935} & \textbf{0.009} & \underline{27.37} & \underline{0.896} & \textbf{0.027} \\
\bottomrule
\end{tabular}
}
\caption{Quantitative evaluation of image restoration on FFHQ.}
\label{table:inverse}
\end{center}
\vspace{-0.5cm}
\end{table}

\subsection{Fast Image Restoration} \label{sec:restoration}

\begin{wrapfigure}{r}{0.3\textwidth}
\vspace{-1cm}
\centering
\includegraphics[width=0.3\textwidth]{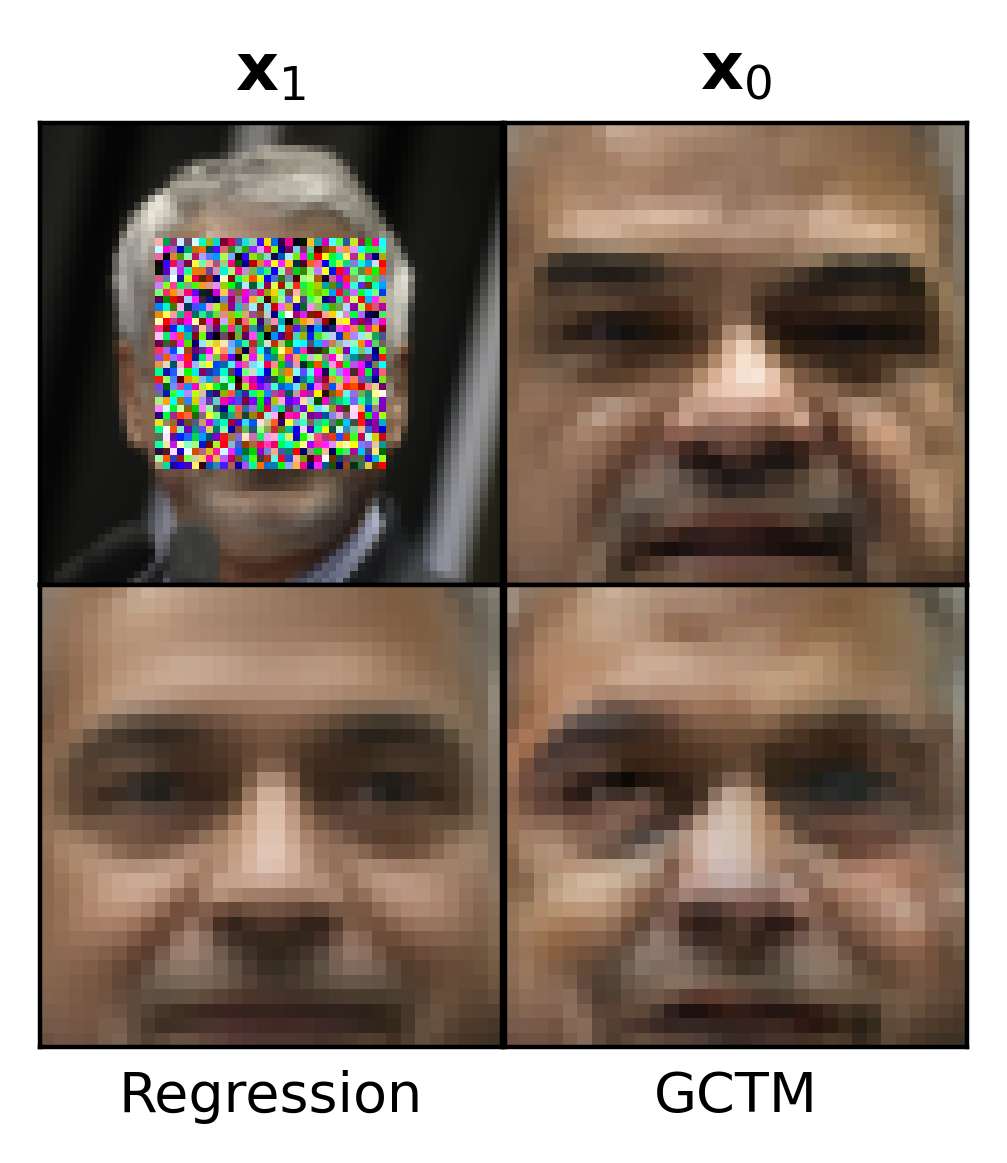}
\vspace{-0.5cm}
\caption{Reg. vs. GCTM.}
\label{fig:inv_sample}
\vspace{-0.7cm}
\end{wrapfigure}

We consider two settings on the FFHQ $64 \times 64$ dataset, where we either know or do not know the corruption operator. In the former case, we train an unconditional GCTM with the independent coupling, with which we implement three zero-shot image restoration algorithms: DPS, CM-based image restoration, and the guided generation algorithm illustrated in Figure \ref{fig:main}, where the loss is given as inconsistency between observations (see Append. \ref{append:algo-inv} for pseudo-codes and a detailed discussion of the differences). In the latter case, we train a GCTM with the supervised coupling and $\ell_2$-regression, \iisb \, and Palette for comparison. Notably, GCTM is the only model applicable to both situations, thanks to the flexible choice of couplings. We again control NFEs such that all methods have similar inference speed.

Table \ref{table:inverse} presents the numerical results in both settings. In the zero-shot setting, we see GCTM outperforming both DPS and CM. In particular, CM is slightly worse than DPS. Sample quality degradation due to error accumulation for CMs at large NFEs have already been observed in unconditional generation (\eg, see Fig. 9 in \citep{kim2023consistency}), and we speculate a similar problem occurs for CMs in image restoration as well.
On the other hand, GCTMs avoid this problem, as they are able to traverse to a smaller time using the ODE velocity approximated via $g_\theta$.
In the supervised setting, we see regression attains the best PSNR and SSIM. This is a natural consequence of perception-distortion trade-off. Specifically, regression minimizes the MSE loss, so it leads to best distortion metrics \citep{delbracio2023indi} while producing blurry results. GCTM, which provides best results if we exclude regression on distortion metrics (PSNR and SSIM) and best results on perception metrics (LPIPS), strikes the best balance between perception and distortion. For instance, in Fig. \ref{fig:inv_sample} inpainting results, regression sample lacks detail (e.g., wrinkles) while GCTM sample is sharp. We show more samples in Appendix \ref{append:exp}.
In particular, in Table \ref{table:gctm_256}, we demonstrate image restoration task of GCTM on ImageNet with higher resolution  (256 $\times$ 256 resolution) to demonstrate it scalability.

\begin{figure}[t]
\centering
\includegraphics[width=0.95\linewidth]{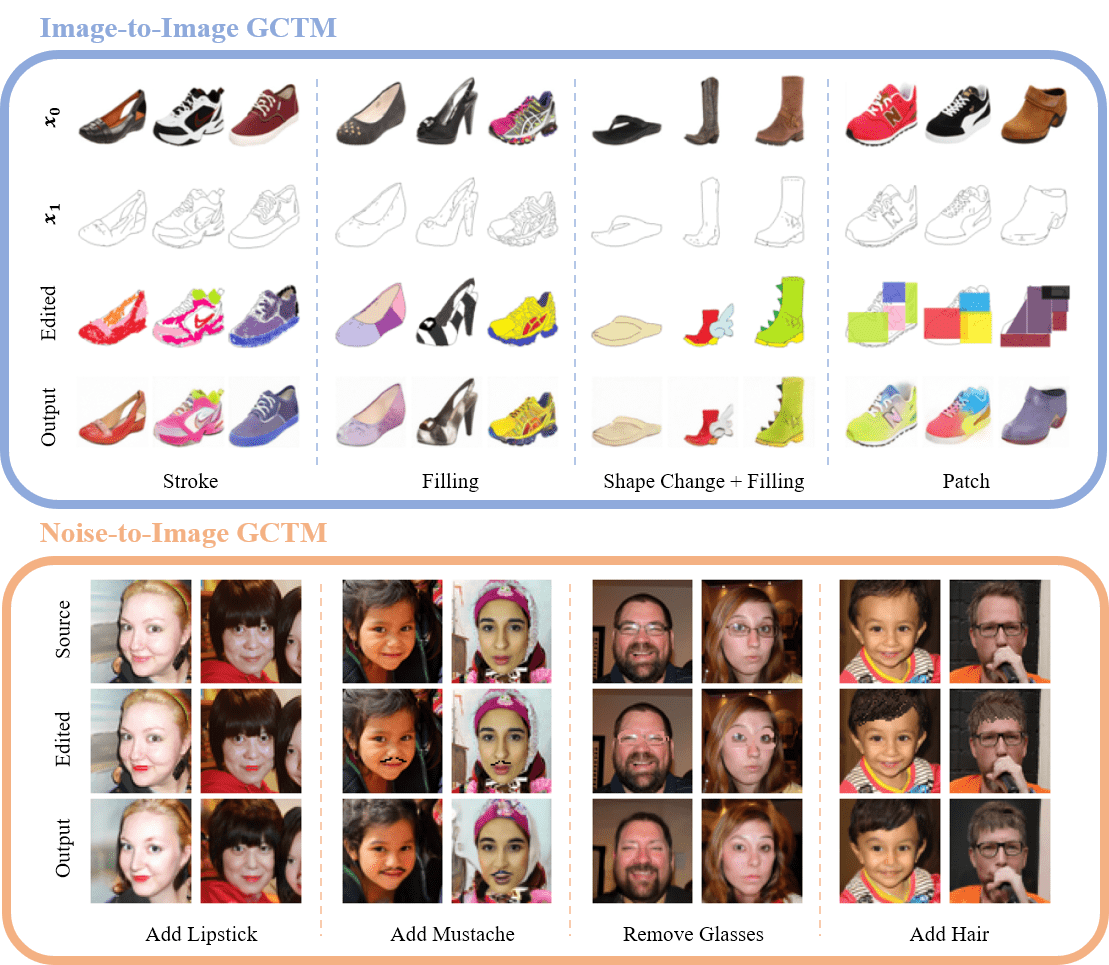}
\caption{Image editing with GCTM, NFE = 1.}
\label{fig:editing}
\vspace{-0.5cm}
\end{figure}

\subsection{Fast Image Editing} \label{sec:editing}

In this section, we demonstrate that GCTM can perform realistic and faithful image editing without any special purpose training. Figure \ref{fig:editing} shows image editing with an Edges$\rightarrow$Shoes model and an unconditional FFHQ model. On Edges$\rightarrow$Shoes, to edit an image, a user creates an edited input, which is an edge image painted to have a desired color and / or modified to have a desired outline. We then interpolate the edited input and the original edge image to a certain time point $t = s$ and send it to time $t = 0$ with GCTM to produce the output. On FFHQ, analogous to SDEdit \citep{meng2021sdedit}, we interpolate an edited image with Gaussian noise and send it to time $t = 0$ with GCTM to generate the output. In contrast to previous image editing models such as SDEdit, GCTM requires only a single step to edit an image. Moreover, we observe that GCTM faithfully preserves source image structure while making the desired changes to the image.

\begin{figure}[t]
\centering
\begin{subfigure}{0.6\linewidth}
\includegraphics[width=1.0\linewidth]{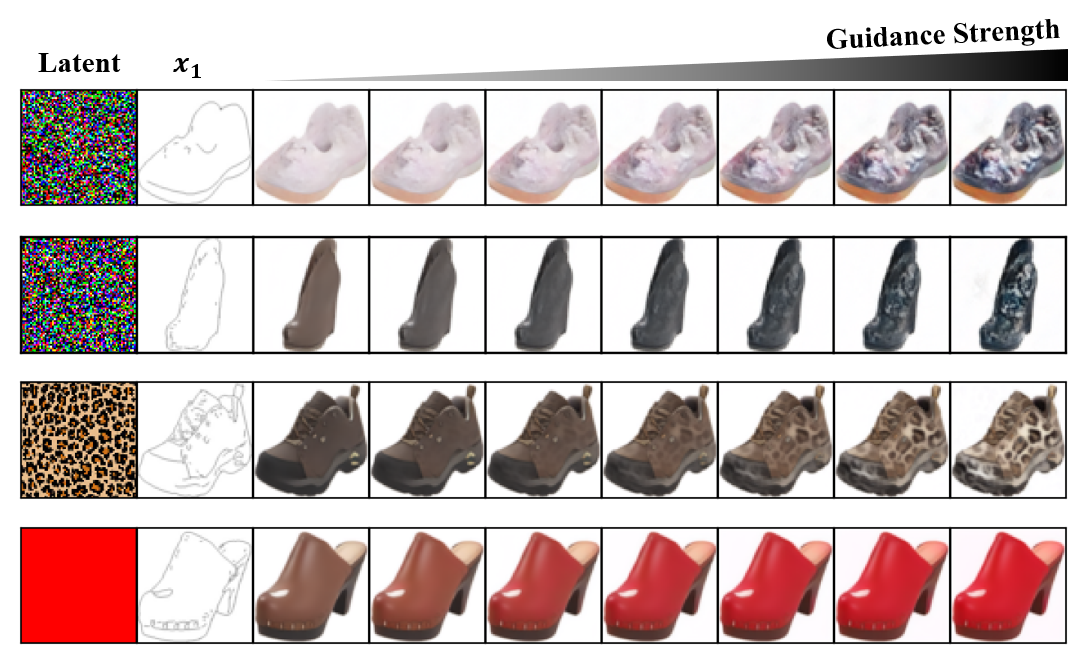}
\caption{Controlling latent vector strength $\gamma$}
\end{subfigure}
\hfill
\begin{subfigure}{0.35\linewidth}
\includegraphics[width=1.0\linewidth]{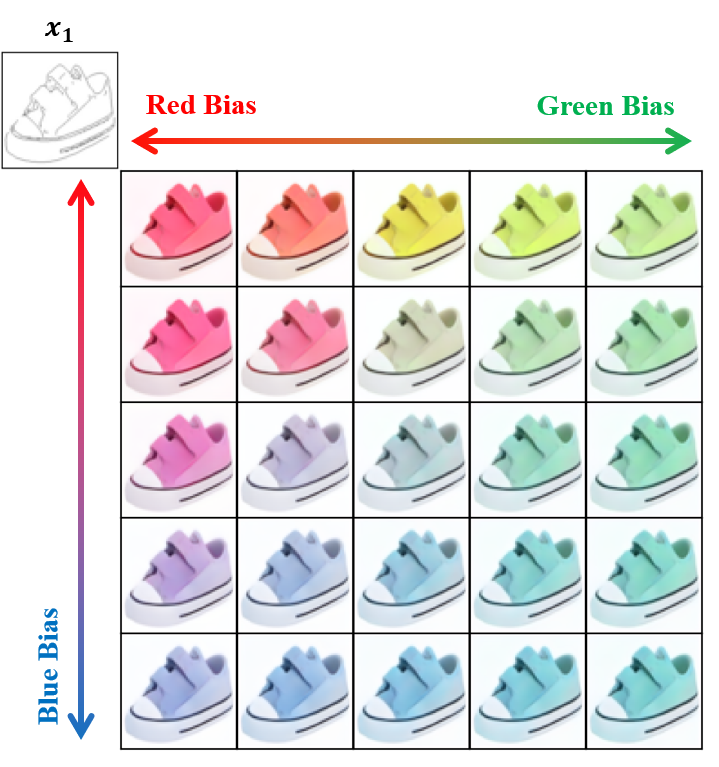}
\caption{Mixing latent vectors}
\end{subfigure}
\caption{Latent manipulation with image-to-image GCTM, NFE = 1.}
\label{fig:latent}
\vspace{-0.5cm}
\end{figure}

\subsection{Fast Latent Manipulation}

In this section, we demonstrate that GCTMs have a highly controllable latent space. Since there are plenty of works on latent manipulation with unconditional diffusion models, we focus on latent manipulation with GCTMs trained for image-to-image translation. For an image-to-image translation GCTM trained with Gaussian perturbation in Section \ref{sec:design}, we assert that the perturbation added to $\xx_1$ can be manipulated to produce desired outputs $\xx_0$. In other words, the perturbation acts as a ``latent vector'' which controls the factors of variation in $\xx_0$. To test this hypothesis, in Figure \ref{fig:latent}, we display outputs $G_\theta(\xx_1 + \gamma \bm{\epsilon}, 1, 0)$ for particular choices of $\bm{\epsilon}$. In the left panel, we observe generated outputs increasingly adhere to the texture of latent $\bm{\epsilon}$ as we increase guidance strength $\gamma$. Interestingly, GCTM generalizes well to latent vectors unseen during training, such as leopard spots or the color red. In the right panel, we explore the effect of linearly combining red, green, and blue latent vectors. We see that the desired color change is reflected faithfully in the outputs. These observations validate our hypothesis that image-to-image GCTMs have an interpretable latent space.

\subsection{Ablation Study}

\begin{wrapfigure}{r}{0.4\textwidth}
\vspace{-0.5cm}
\centering
\includegraphics[width=0.4\textwidth]{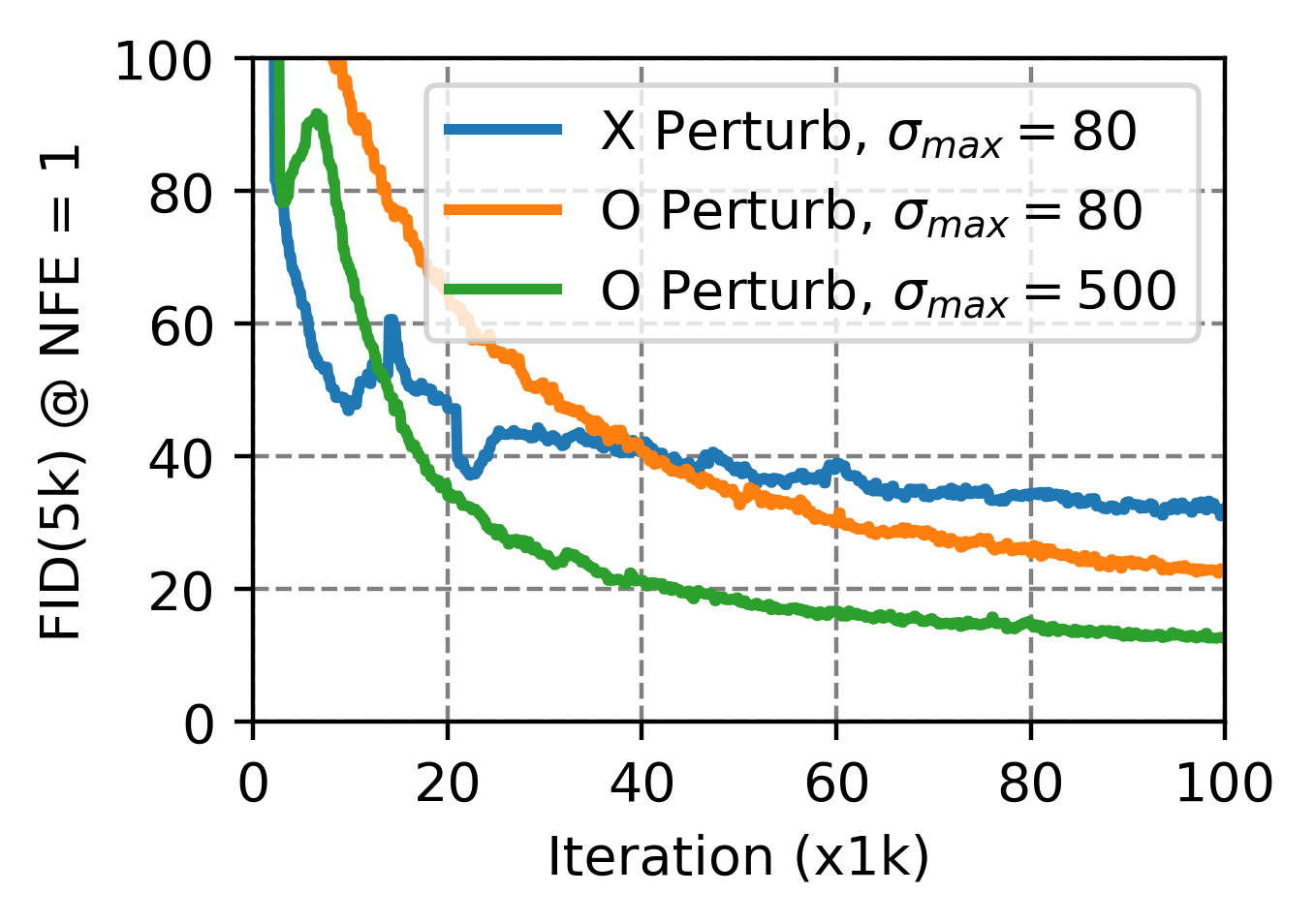}
\caption{Ablation study of GCTM.}
\label{fig:ablation}
\vspace{-0.3cm}
\end{wrapfigure}

We now perform an ablation study on the design choices of Section \ref{sec:design}. We have already illustrated the power of using appropriate couplings in previous sections, so we explore the importance of $\sigma_{\max}$. A robust choice for $\sigma_{\max}$ for unconditional generation is well-known to be $\sigma_{\max} = 80$ \citep{karras2022edm,kim2023consistency}, and we found using this choice to perform sufficiently well for GCTMs when learning to translate noise to data with independent or OT couplings. So, we restrict our attention to image-to-image translation.

In Figure \ref{fig:ablation}, we display the learning curves on Edges$\rightarrow$Shoes for GCTMs trained without and with Gaussian perturbation, and $\sigma_{\max} \in \{80,500\}$. We observe that GCTM trained without perturbation and $\sigma_{\max} = 80$ exhibits unstable dynamics, and is unable to minimize the FID below $30$. On other hand, GCTM trained with perturbation and $\sigma_{\max} = 80$ surpasses the model trained without perturbation. This demonstrates Gaussian perturbation is indeed crucial for one-to-many generation, as noted in the last paragraph of Section \ref{sec:design}. Finally, GCTM with both perturbation and $\sigma_{\max} = 500$ minimizes FID the fastest. This shows high-curvature regions for image-to-image ODEs lie near $\xx_1$, so we need to use a large $\sigma_{\max}$ which places more discretization points near $t = 1$.

\section{Conclusion}

Our work marks a significant advancement in the realm of ODE-based generative models, particularly on the transformative capabilities of Consistency Trajectory Models (CTMs). While the iterative nature of diffusion has proven to be a powerful foundation for high-quality image synthesis and nuanced control, the computational demands associated with numerous neural function evaluations (NFEs) per sample have posed challenges for practical implementation. Our proposal of Generalized CTMs (GCTMs) extends the reach of CTMs by enabling one-step translation between arbitrary distributions, surpassing the limitations of traditional CTMs confined to Gaussian noise to data transformations. Through an insightful exploration of the design space, we elucidate the impact of various components on downstream task performance, providing a comprehensive understanding that contributes to a broadly applicable and stable training scheme. Empirical validation across diverse image manipulation tasks demonstrates the potency of GCTMs, showcasing their ability to accelerate and enhance diffusion-based algorithms. In summary, our work not only contributes to theoretical advancements but also delivers tangible benefits, showcasing GCTMs as a key element in unlocking the full potential of diffusion models for practical, real-world applications in image synthesis, translation, restoration, and editing.



\bibliographystyle{iclr2025_conference}

\newpage
\appendix

\section{Full Experiment Settings} \label{append:settings}

\subsection{Training}

In this section, we introduce training choices which provided reliable performance across all experiments in our paper.

\noindent
\textbf{Bootstrapping scores.} In all our experiments, we train GCTMs without a pre-trained score model. So, analogous to CTMs, we use velocity estimates given by an exponential moving average $\theta_{\EMA}$ of $\theta$ to solve ODEs. We use exponential moving average decay rate $0.999$.

\noindent
\textbf{Time discretization.} In practice, we discretize the unit interval into a finite number of timesteps $\{t_n\}_{n = 0}^N$ where
\begin{align}
t_0 = 0 < t_1 < \cdots < t_N = 1
\end{align}
and learn ODE trajectories integrated with respect to the discretization schedule. EDM \citep{karras2022edm}, which has shown robust performance on a variety of generation tasks, solves the PFODE on the time interval $(\sigma_{\min},\sigma_{\max})$ for $0 < \sigma_{\min} < \sigma_{\max}$ according to the discretization schedule
\begin{align}
\sigma_n = (\sigma_{\min}^{1/\rho} + (n/N) (\sigma_{\max}^{1/\rho} - \sigma_{\min}^{1/\rho}))^\rho
\end{align}
for $n = 0, \ldots, N$ and $\rho = 7$. Thus, using the change of time variable \eqref{eq:change} derived in Theorem \ref{theorem:1}, we convert PFODE EDM schedule to FM ODE discretization
\begin{align}
t_0 = 0, \quad t_n = \sigma_n / (1 + \sigma_n) \quad \text{for} \quad n = 1, \ldots, N-1, \quad t_{N} = 1.
\end{align}
In our experiments, we fix $\sigma_{\min} = 0.002$ and control $\sigma_{\max}$. We note that $\sigma_{\max}$ controls the amount of emphasis on time near $t = 1$, i.e., larger $\sigma_{\max}$ places more time discretization points near $t = 1$.

\noindent
\textbf{Number of discretization steps $N$.} CTMs use fixed $N = 18$. In contrast, analogous to iCMs, we double $N$ every $100k$ iterations, starting from $N = 4$.

\noindent
\textbf{Time $\hat{t}$ distribution.} For unconditional generation, we sample
\begin{align}
\hat{t} = \sigma / (1 + \sigma), \qquad \log \sigma \sim \mathcal{N}(-1.2, 1.2^2)
\end{align}
in accordance with EDM. For image-to-image translation, we sample
\begin{align}
\hat{t} \sim \text{beta}(3,1).
\end{align}


\noindent
\textbf{Network conditioning.} We use the EDM conditioning, following CTMs.

\noindent
\textbf{Distance $d$.} CTMs use $d$ defined as
\begin{align}
d(\xx_t,\hat{\xx}_t) = \LPIPS(G_{\theta_{\EMA}}(\xx_t,t,0),G_{\theta_{\EMA}}(\hat{\xx}_t,t,0))
\end{align}
which compares the perceptual distance of samples projected to time $t = 0$. In contrast, following iCMs, we use the pseudo-huber loss
\begin{align}
d(\xx_t,\hat{\xx}_t) = \sqrt{\|\xx_t - \hat{\xx}_t\|_2^2 + c^2} - c
\end{align}
where $c = 0.00054 \sqrt{d}$, where $d$ is the dimension of $\xx_t$.

\noindent
\textbf{Batch size.} We use batch size $128$ for $32 \times 32$ resolution images and batch size $64$ for $64 \times 64$ resolution images.

\noindent
\textbf{Optimizer.} We use the Adam optimizer \citep{kingma2015adam} with learning rate
\begin{align}
\eta = 0.0002 / (128 / \text{\texttt{batch\_size}})
\end{align}
and default $(\beta_1,\beta_2) = (0.9, 0.999)$.

\noindent
\textbf{Coefficient for $\LL_{\FM}(\theta)$.} We use $\lambda_{\FM} = 0.1$ for all experiments.

\noindent
\textbf{Network.} We modify \texttt{SongUNet} provided at \url{https://github.com/NVlabs/edm} to accept two time conditions $t$ and $s$ by using two time embedding layers.

\noindent
\textbf{ODE Solver.} We use the second order Heun solver to calculate $\LL_{\GCTM}(\theta)$.

\noindent
\textbf{Gaussian perturbation.} We apply a Gaussian perturbation from a normal distribution multiplied by 0.05 to sample $\xx_1$, excluding inpainting task.

\subsection{Evaluation}

In this section, we describe the details of the evaluation to ensure reproducibility of our experiments.

\noindent
\textbf{Datasets.} In unconditional generation task, we compare our GCTM generation performance using CIFAR10 training dataset. In image-to-image translation task, we evaluate the performance of models using test sets of Edges$\rightarrow$Shoes, Night$\rightarrow$Day, Facades from Pix2Pix. In image restoration task, we use FFHQ and apply following corruption operators $\HH$ from I$^2$SB to obtain measurement: bicubic super-resolution with a factor of 2, Gaussian deblurring with $\sigma = 0.8$, and center inpainting with Gaussian. We then assess model performance using test dataset.

\noindent
\textbf{Baselines.} In image-to-image translation task, we compare three baselines: Pix2Pix from \url{https://github.com/junyanz/pytorch-CycleGAN-and-pix2pix}, Palette model from \url{https://github.com/Janspiry/Palette-Image-to-Image-Diffusion-Models}, and I$^2$SB from \url{https://github.com/NVlabs/I2SB}. We modify the image resolution to 64 $\times$ 64 and keep the hyperparameters as described in their code bases, except that Pix2Pix due to the input size constraints of the discriminator. Same configuration is used in supervised image restoration task.

\noindent
\textbf{Metrics details.} We calculate FID using \url{https://github.com/mseitzer/pytorch-fid} and IS from \url{https://github.com/pytorch/vision/blob/main/torchvision/models/inception.py}. We assess LPIPS from \url{https://github.com/richzhang/PerceptualSimilarity} with AlexNet version 0.1. In generation task, we employ the entire training dataset to obtain FID scores, and in the other task, we sample 5,000 test datasets. To obtain PSNR and SSIM, we convert the data type of model output to \texttt{uint8} and normalize it. We use \url{https://github.com/scikit-image/scikit-image} for PSNR and SSIM.

\noindent
\textbf{Sampling time.} To compare inference speed, we measure the average time between the model taking in one batch size as input and outputting it.

\newpage

\section{Algorithms} \label{append:algo}

\subsection{Optimal Transport}
\begin{algorithm}[h]
\caption{Sinkhorn-Knopp (SK)}
\begin{algorithmic}[1]
\State \textbf{Input:} $\{\xx_0^m\}_{m=1}^M$, $\{\xx_1^m\}_{m=1}^M$, $\tau$
\State Compute cost matrix $\bm{C} \in \mathbb{R}^{M \times M}$ such that $C_{i,j} = \|\xx_0^i - \xx_1^j\|_2^2$
\State With Alg. 1 in \citep{cuturi2013sinkhorn}, solve $\bm{P}^{\OT} = \argmin_{\bm{P}} \langle \bm{P}, \bm{C} \rangle - \tau H(\bm{P})$ s.t. $\bm{P} \bm{1} = \bm{P}^\top \bm{1} = \frac{1}{n} \bm{1}$
\State Treat $\bm{P}^{\OT}$ as a discrete distribution over $\{1,\ldots,M\} \times \{1,\ldots,M\}$
\State Sample $\{(i^m,j^m)\}_{m=1}^M \sim \bm{P}^{\OT}$
\State \textbf{Return:} $\{(\xx_0^{i^m},\xx_1^{j^m})\}_{m=1}^M$
\end{algorithmic}
\label{alg:sk}
\end{algorithm}

\subsection{Image Restoration} \label{append:algo-inv}

\begin{algorithm}[h]
\caption{Zero-shot Image Restoration}
\begin{algorithmic}[1]
\State \textbf{Input:} Measurement $\xx_1$, corruption $\HH$, discretization $\{t_i\}_{i=0}^{M}$
\State $\xx_{t_M}' \sim \mathcal{N}(0,\bm{I})$
\For{$i = M$ \textbf{to} 1}
\State $\bm{\epsilon} \sim \mathcal{N}(0,\bm{I})$
\If{Method is \texttt{DPS}} 
\State $\hat{\xx}_0 = g_{\theta}(\xx_{t_i}',t_i,t_i)$
\State $\xx_{t_{i-1}}' = (1 - t_{i-1}) \hat{\xx}_0 + t_{i-1} \bm{\epsilon}$
\ElsIf{Method is \texttt{CM}} 
\State $\hat{\xx}_0 = G_{\theta}(\xx_{t_i}', t_i, 0)$
\State $\xx_{t_{i-1}}' = (1 - t_{i-1}) \hat{\xx}_0 + t_{i-1} \bm{\epsilon}$
\ElsIf{Method is \texttt{GCTM}} 
\State Evaluate score and ODE endpoint in parallel by $t = (t_i,t_i)$, $s = (t_i,0)$ :
\State $\widetilde{\xx}_0, \hat{\xx}_0 = g_{\theta}(\xx_{t_i}',t_i,t_i), G_{\theta}(\xx_{t_i}',t_i,0)$
\State $\xx_{t_{i-1}}' = (1 - t_{i-1}) \widetilde{\xx}_0 + t_{i-1} \bm{\epsilon}$
\EndIf
\State $\xx_{t_{i-1}}' \gets \xx_{t_{i-1}}' - \lambda\nabla_{\xx_{t_i}'}||\xx_1 - \HH \hat{\xx}_0||^2_2$
\EndFor
\State \textbf{Return:} $\xx_0'$
\end{algorithmic}
\label{alg:image_restoration}
\end{algorithm}

In Alg. \ref{alg:image_restoration}, we describe three zero-shot image restoration algorithms, DPS, CM, and GCTM. DPS uses the posterior mean $\EE_{q(\xx_0 | \xx_{t_i}')}[\xx_0]$ to both traverse to a smaller time $t_{i-1}$ and to approximate measurement inconsistency. As the posterior mean generally do not lie in the data domain, using it to calculate measurement inconsistency can be problematic. Indeed, approximation error in DPS is closely related to the discrepancy between the posterior mean and $\xx_{t_i \rightarrow 0}'$ (see Theorem 1 in \citep{chung2022dps} for a formal statement). On the other hand, CM uses the ODE terminal point $\xx_{t_i \rightarrow 0}'$ to traverse to a smaller time $t_{i-1}$ and to approximate measurement inconsistency. While CM can have better guidance gradients as $\xx_{t_i \rightarrow 0}'$ lie within the data domain, using $\xx_{t_i \rightarrow 0}'$ to traverse to $t_{i-1}$ can accumulate truncation error and degrade sample quality. For instance, see Figure 9 (a) in \citep{kim2023consistency}. GCTM mitigates both problems by enabling parallel evaluation of posterior mean and ODE endpoint, as shown in Line 12-13 of Alg. \ref{alg:image_restoration}.

\subsection{Image Editing} \label{append:algo-edit}

\begin{algorithm}[h]
\caption{Image Editing}
\begin{algorithmic}[1]
\State \textbf{Input:} $(\xx_0, \xx_1) \sim q(\xx_0,\xx_1)$, $t$
\State $\hat{\xx}_t = (1-t) \text{Edit}(\xx_0) + t\xx_1$
\State \textbf{Return:} $G_\theta(\hat{\xx}_t,t,0)$
\end{algorithmic}
\label{alg:image_editing}
\end{algorithm}

\section{Proofs}

\subsection{Proof of Theorem \ref{theorem:1}} \label{proof:theorem1}

\begin{proof}
We observe that the velocity term in \eqref{eq:flow_ode} may be expressed as
\begin{align}
\EE_{q(\xx_0, \xx_1 | \xx_t)}[\xx_1 - \xx_0] &= \EE_{q(\xx_0, \xx_1 | \xx_t)}[(\xx_t - \xx_0) / t] \\
&= \EE_{q(\xx_0 | \xx_t)}[(\xx_t - \xx_0) / t] \\
&= (\xx_t - \EE_{q(\xx_0 | \xx_t)}[\xx_0]) / t
\end{align}
since $\xx_1$ is determined given $\xx_0$ and $\xx_t$. This shows the equivalence between \eqref{eq:flow_ode} and \eqref{eq:flow_ode_equiv}. Eqs. \eqref{eq:flow_G_param} and \eqref{eq:flow_g} are straightforward consequences of the equivalence between ODEs.
\end{proof}

\subsection{Proof of Theorem \ref{theorem:2}} \label{proof:theorem2}

\subsubsection{Proof of part (i)}

\begin{proof}
We first show equivalence of scores. We note that
\begin{align}
\xx_t \mapsto \bar{\xx}_{t'}
\end{align}
is a bijective transformation, so by change of variables,
\begin{align}
q(\bar{\xx}_{t'} | \xx_0) = (1 + t) \cdot \mathcal{N}(\xx_t | \xx_0, t \bm{I}) = (1 + t) \cdot p(\xx_t | \xx_0)
\end{align}
and marginalizing out $\xx_0$, we get
\begin{align}
q(\bar{\xx}_{t'}) = (1 + t) \cdot p(\xx_t).
\end{align}
It follows by Bayes' rule that
\begin{align}
p(\xx_0 | \xx_t) &= \frac{p(\xx_t | \xx_0) p(\xx_0)}{p(\xx_t)} \\
&= \frac{(1 + t)^{-1} q(\bar{\xx}_{t'} | \xx_0) q(\xx_0)}{(1 + t)^{-1} q(\bar{\xx}_{t'})} \\
&= \frac{q(\bar{\xx}_{t'} | \xx_0) q(\xx_0)}{q(\bar{\xx}_{t'})} \\
&= q(\xx_0 | \bar{\xx}_{t'})
\end{align}
and thus
\begin{align}
\EE_{p(\xx_0 | \xx_t)}[\xx_0] = \EE_{q(\xx_0 | \bar{\xx}_{t'})}[\xx_0].
\end{align}
for all $t \in (0,\infty)$ and $\xx_t$. We now show equivalence of ODEs. Let us first re-state the diffusion PFODE below.
\begin{align}
d\xx_t = \frac{\xx_t - \EE_{p(\xx_0 | \xx_t)}[\xx_0]}{t} \, dt.
\end{align}
With the change of variable
\begin{align}
\bar{\xx}_t = \xx_t / (1 + t),
\end{align}
we have
\begin{align}
d \bar{\xx}_t &= - \frac{\xx_t}{(1 + t)^2} \, dt + \frac{1}{1 + t} \, d\xx_t \\
&= - \frac{\xx_t}{(1 + t)^2} \, dt + \frac{\xx_t - \EE_{p(\xx_0 | \xx_t)}[\xx_0]}{t(1 + t)} \, dt \\
&= - \frac{\bar{\xx}_t}{1 + t} \, dt + \frac{(1 + t)\bar{\xx}_t - \EE_{p(\xx_0 | \xx_t)}[\xx_0]}{t(1 + t)} \, dt \\
&= \frac{\bar{\xx}_t - \EE_{p(\xx_0 | \xx_t)}[\xx_0]}{t(1 + t)} \, dt \\
&= \frac{\bar{\xx}_t - \EE_{q(\xx_0 | \bar{\xx}_{t'})}[\xx_0]}{t(1 + t)} \, dt
\end{align}
where we have used equivalence of scores at the last line. We then make the change of time variable
\begin{align}
t' = t / (1 + t) \implies dt' = \frac{1}{(1 + t)^2} \, dt
\end{align}
which gives us
\begin{align}
d \bar{\xx}_{t'} &= \frac{\bar{\xx}_{t'} - \EE_{q(\xx_0 | \bar{\xx}_{t'})}[\xx_0]}{t/(1 + t)} \, dt' \\
&= \frac{\bar{\xx}_{t'} - \EE_{q(\xx_0 | \bar{\xx}_{t'})}[\xx_0]}{t'} \, dt'.
\end{align}
This concludes the proof.
\end{proof}

\subsubsection{Proof of part (ii)}

\begin{proof}
For the first equality in \eqref{eq:equality}, transform PFODE variables $(\xx_t,t)$ into FM ODE variables $(\bar{\xx}_{t'},t')$ with \eqref{eq:change}, transport $\bar{\xx}_{t'}$ to $\bar{\xx}_{s'}$ with $G_{\GCTM}$, and then transform FM ODE variables $(\bar{\xx}_{s'},s')$ into PFODE variables $(\xx_s,s)$ with the inverse of \eqref{eq:change}. Second equality in \eqref{eq:equality} follows directly from \eqref{eq:score_equiv}.
\end{proof}


\section{Limitation, Social Impacts, and Reproducibility}

\noindent
\textbf{Limitations.} GCTMs are yet unable to reach state-of-the-art unconditional generative performance. We speculate further tuning of hyper-parameters in the manner of iCMs could improve the performance, and leave this for future work.

\noindent
\textbf{Social impacts.} GCTM generalizes CTM to achieve fast translation between any two distributions. Hence, GCTM may be used for beneficial purposes, such as fast medical image restoration. However, GCTM may also be used for malicious purposes, such as generation of malicious images, and this must be regulated.

\noindent
\textbf{Reproducibility.} We will open-source our code upon acceptance. Moreover, we have submitted experiment codes as a supplementary material.

\section{Additional Experiments} \label{append:exp}

\subsection{Comparing I2I performance with other baseline models}

We compare the image-to-image (I2I) performance of our model with two baseline approaches: EGSDE~\citep{zhao2022egsde} and BBDM~\citep{li2023bbdm}. Since BBDM, an I2I framework based on the Brownian Bridge process, leverages a latent diffusion model, we train it with a pixel-space diffusion model for a fair comparison. Both BBDM and EGSDE are trained on the Edges$\rightarrow$Shoes dataset. As shown in Table \ref{tab:edit_base}, our GCTM outperforms all baselines across various metrics, even when evaluated with fewer sampling steps.

In addition, we visualize the image editing results in Fig. \ref{fig:editing_compare}. While EGSDE generates realistic images, it fails to faithfully preserve the given conditions. BBDM, on the other hand, struggles to perform robustly on (unseen) conditional images. In contrast, GCTM produces realistic images while accurately maintaining the original conditions.

\begin{table}[!b]
\caption{Evaluation of I2I translation on Edges$\rightarrow$Shoes with other baslines.}
\centering
\begin{tabular}{cccccc}
\toprule
\textbf{Method} & \textbf{NFE} & \textbf{Time (ms)} & \textbf{FID} $\downarrow$ & \textbf{IS} $\uparrow$ & \textbf{LPIPS} $\downarrow$ \\
\cmidrule{1-6}
BBDM~\citep{li2023bbdm} & 5 & 75 & 43.7 & 3.43 & 0.099 \\ 
EGSDE~\citep{zhao2022egsde} & 500 & 2590 & 198.1 & 2.87 & 0.476 \\
GCTM & 1 & 87 & \textbf{40.3} & \textbf{3.54} & \textbf{0.097} \\
\bottomrule
\label{tab:edit_base}
\end{tabular}
\end{table}

\begin{figure}[t]
    \centering
    \includegraphics[width=\linewidth]{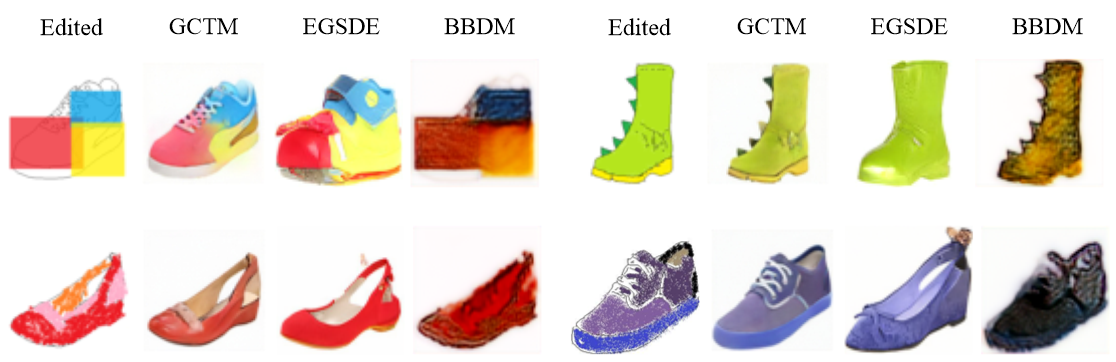}
    \caption{Comparison on image editing with GCTM and other baselines}
    \label{fig:editing_compare}
\end{figure}

\subsection{Controllable Image Editing}

In this section, we demonstrate that effectiveness of image editing can be controlled. In Algorithm \ref{alg:image_editing}, we control the time point $t$ to determine how much of the edited image to reflect. In Fig. \ref{fig:editing_time}, the results visualize how $t$ effect the output of model output. We observe that the larger $t$, the more realistic the image, and the smaller $t$, the more faithful the edit feature. We set $t = 0.95$ and $t = 0.4$ at supervised coupling and independent coupling, respectively. 

\begin{figure}[!ht]
    \centering
    \includegraphics[width =\textwidth]{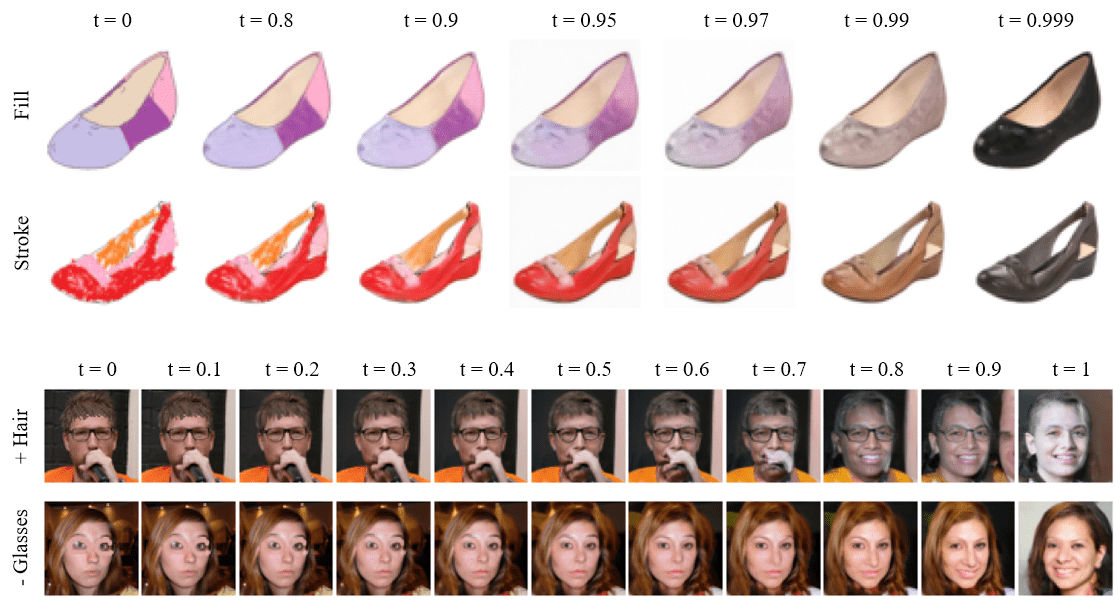}
    \caption{Controllability of image editing by $t$.}
    \label{fig:editing_time}
\end{figure}

\subsection{High-resolution Image restoration}

In Table \ref{table:gctm_256}, we demonstrate image restoration task of GCTM on ImageNet with higher resolution. 

\begin{table}[!h]
\caption{GCTM evaluation of image restoration on ImageNet with 256 $\times$ 256 resolution.}
\vspace{-0.3cm}
\begin{center}
\renewcommand{\arraystretch}{1}
\resizebox{\textwidth}{!}
{\small
\begin{tabular}{ccccccccccccc}
\toprule
 \multirow{2}{*}{Method} & \multirow{2}{*}{NFE}
 &\multicolumn{3}{c}{SR4 - Bicubic} &  \multicolumn{3}{c}{Deblur - Gaussian} & \multicolumn{3}{c}{Inpaint - Center}
  \\ \cmidrule(lr){3-5} \cmidrule(lr){6-8} \cmidrule(lr){9-11}
 \multicolumn{1}{c}{} & \multicolumn{1}{c}{} &
  \multicolumn{1}{c}{PSNR $\uparrow$} &
  \multicolumn{1}{c}{SSIM $\uparrow$ } &
  \multicolumn{1}{c}{LPIPS $\downarrow$ } &
  \multicolumn{1}{c}{PSNR $\uparrow$} &
  \multicolumn{1}{c}{SSIM $\uparrow$ } &
  \multicolumn{1}{c}{LPIPS $\downarrow$ } &
  \multicolumn{1}{c}{PSNR $\uparrow$} &
  \multicolumn{1}{c}{SSIM $\uparrow$ } &
  \multicolumn{1}{c}{LPIPS $\downarrow$ } \\
\cmidrule{1-11}
 Corrupt & - & 24.48& 0.708& 0.340 & 25.26& 0.830& 0.223& 12.92& 0.708& 0.598\\ 
 GCTM & 1 & {26.70} & {0.771} & {0.223} & {34.65} & {0.948} & {0.032} & {21.56} & {0.808} & {0.229} \\
\bottomrule
\end{tabular}
}
\label{table:gctm_256}
\end{center}
\end{table}

\newpage

\subsection{Additional Image-to-Image Translation Samples}

\begin{figure}[h!]
\centering
\small{
\includegraphics[width = \textwidth]{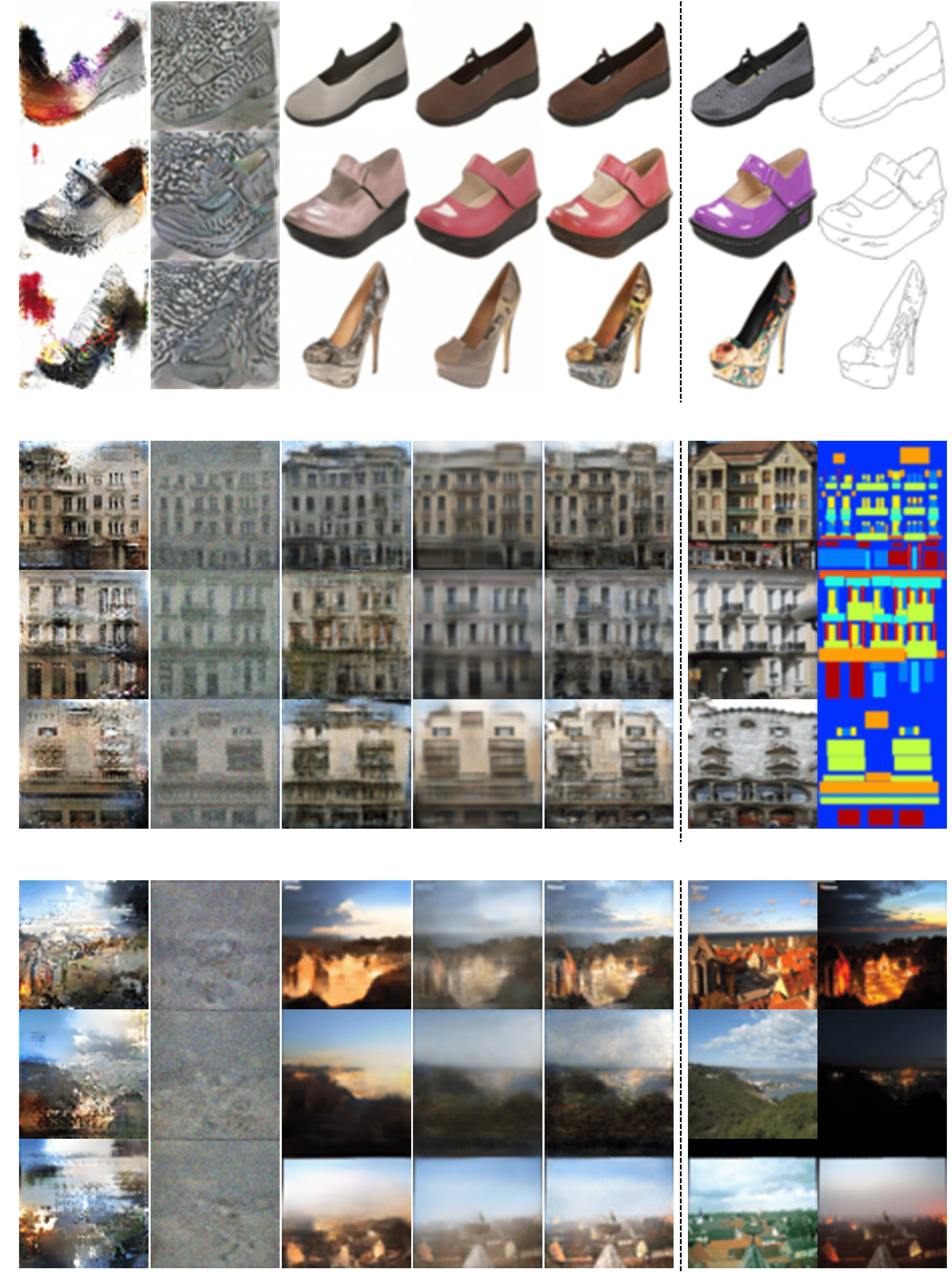}
\put(-376,-7){Pix2Pix}
\put(-320,-7){Palette}
\put(-263,-7){\iisb }
\put(-218,-7){Regression}
\put(-157,-7){GCTM}
\put(-87,-7){$\xx_0$}
\put(-35,-7){$\xx_1$}
}
\caption{Additional results on image-to-image translation task on Edges$\rightarrow$Shoes (top), Night$\rightarrow$Day (middle) and Facades (bottom).}
\label{fig:app_i2i}
\end{figure}

\subsection{Additional Image Restoration Samples}

\begin{figure}[h!]
    \centering
    \includegraphics[width=\textwidth]{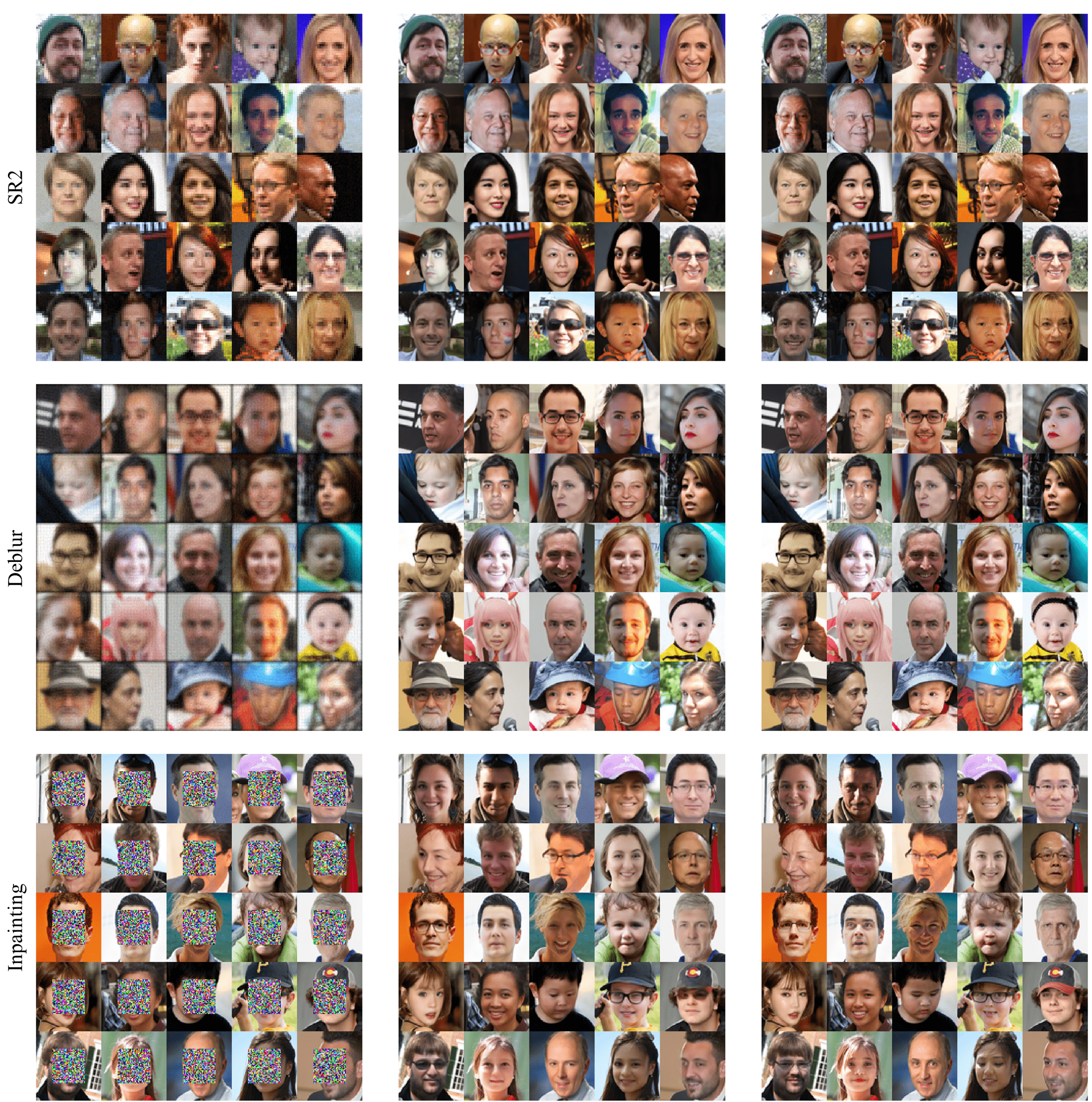}
    \put(-330, -10){$\xx_1$}
    \put(-209, -10){GCTM}
    \put(-65, -10){$\xx_0$}
    \caption{Additional results of supervised image restoration task on FFHQ 64$\times$64.}
    \label{fig:app+i2i2}
\end{figure}

\begin{figure}
    \centering
    \includegraphics[width=\textwidth]{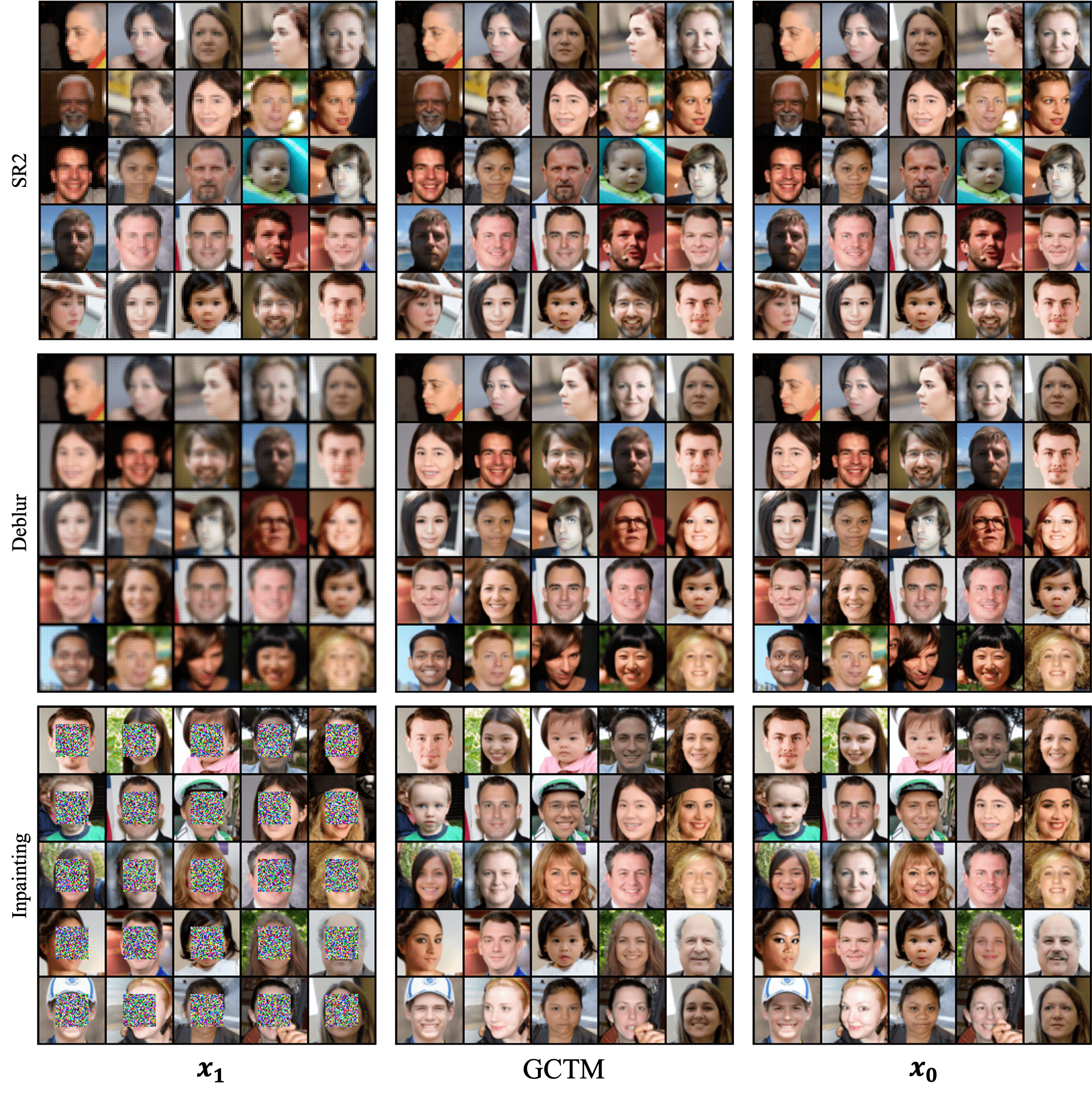}
    \caption{Additional results of zero-shot image restoration task on FFHQ 64$\times$64.}
    \label{fig:zeroshot}
\end{figure}

\begin{figure}
    \centering
    \includegraphics[width=\textwidth]{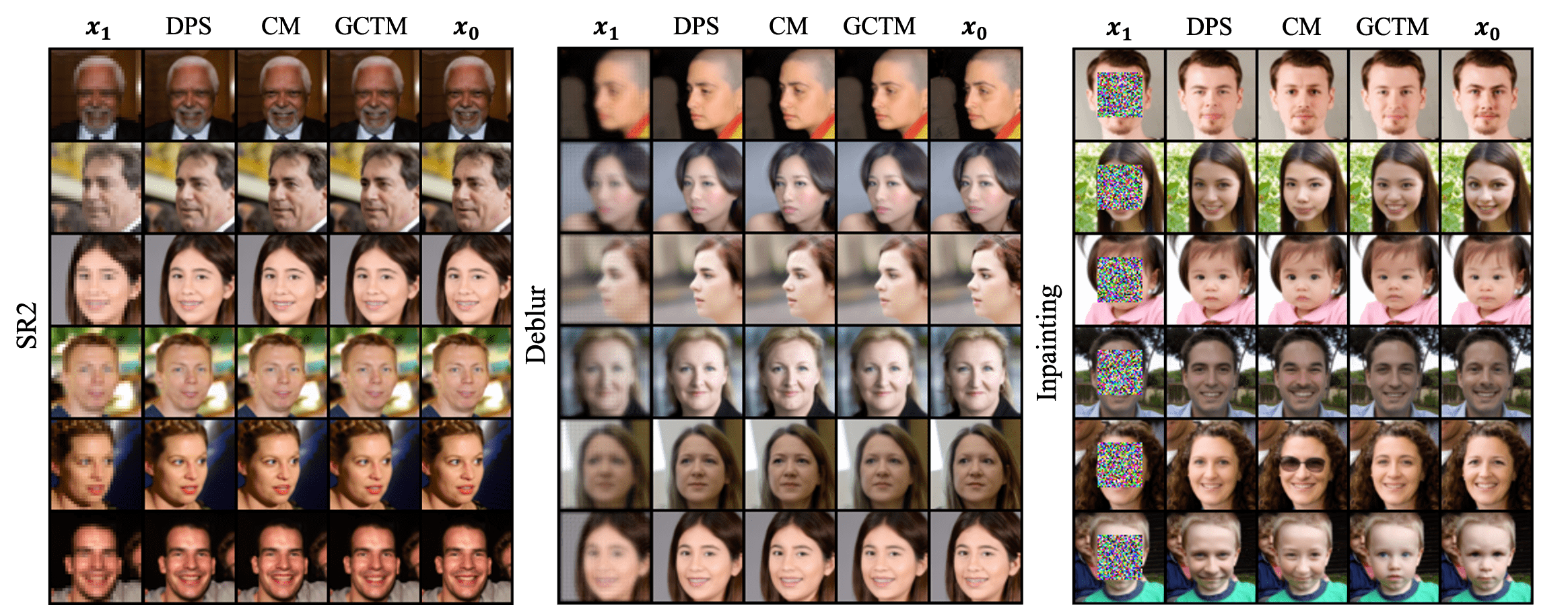}
    \caption{Qualitative comparison of zero-shot algorithms on FFHQ 64$\times$64.}
    \label{fig:zeroshot_compare}
\end{figure}

\end{document}

%% file: math_commands.tex
\usepackage{amsmath,amsfonts,bm}

\def\eqref#1{equation~\ref{#1}}

\def\1{\bm{1}}

\def\vv{{\bm{v}}}

\DeclareMathAlphabet{\mathsfit}{\encodingdefault}{\sfdefault}{m}{sl}
\SetMathAlphabet{\mathsfit}{bold}{\encodingdefault}{\sfdefault}{bx}{n}

\DeclareMathOperator*{\argmin}{arg\,min}